\documentclass[11pt]{article}

\def \citep{\cite}

\usepackage[margin = 1in]{geometry}

\usepackage{palatino}
\usepackage{amsmath}
\usepackage[english]{babel}
\usepackage[toc,page]{appendix} 
\usepackage{hyperref}
\usepackage{graphicx,wrapfig,lipsum}	
\usepackage{yfonts}	 		
\usepackage{amssymb}
\usepackage{amsmath}
\usepackage{amsthm}
\usepackage{scalerel}
\usepackage{verbatim}
\usepackage[lined,boxed,ruled,norelsize,algo2e,linesnumbered]{algorithm2e}
\usepackage{booktabs}
\newlength{\oldintextsep} % for saving \intextsep temporarily
\usepackage{enumitem}

\usepackage{sidecap} % caption on side of figure

\usepackage{dsfont}
\def \ind{\mathds{1}}
\def \N{\mathbb{N}}
\def \R{\mathbb{R}}
\def \G{\mathbb{G}}
\def \calA{\mathcal{A}}
\def \calB{\mathcal{B}}
\def \calC{\mathcal{C}}
\def \calV{\mathcal{V}}
\def \calH{\mathcal{H}}
\def \calE{\mathcal{E}}
\def \calN{\mathcal{N}}
\def \calS{\mathcal{S}}
\def \calX{\mathcal{X}}
\def \calG{\mathcal{G}}
\def \calF{\mathcal{F}}

\usepackage{algorithm}
\usepackage{algorithmic}
\usepackage{tikz}

\newtheorem{definition}{Definition}
\newtheorem{theorem}{Theorem}
\newtheorem{claim}{Claim}

\newtheorem{corollary}{Corollary}
\newtheorem{proposition}{Proposition}
\newtheorem{conjecture}{Conjecture}

\usepackage{amsthm}
\newcommand{\ggnn}{Recursive Neighborhood Pooling     }
\newcommand{\lb}{\lbrace\!\!\lbrace}
\newcommand{\rb}{\rbrace\!\!\rbrace}

\newcommand{\sj}[1]{\textcolor{blue}{SJ: #1}}

\title{ Counting Substructures with Higher-Order Graph Neural Networks: 
Possibility and Impossibility Results  }

\author{ Behrooz Tahmasebi\\ CSAIL, MIT\\ \url{bzt@mit.edu}   \and Derek Lim\\ CSAIL, MIT \\ \url{dereklim@mit.edu} \and Stefanie Jegelka\\ CSAIL, MIT \\ \url{stefje@mit.edu}
}

\date{}

\begin{document}

\maketitle

\begin{abstract}
  While message passing Graph Neural Networks (GNNs) have become increasingly popular architectures for learning with graphs, recent works have revealed important shortcomings in their expressive power. In response, several higher-order GNNs have been proposed that substantially increase the expressive power, albeit at a large computational cost.
  Motivated by this gap, we explore alternative strategies and lower bounds. In particular, we analyze a new recursive pooling technique of local neighborhoods that allows different tradeoffs of computational cost and expressive power. 
  First, we prove that this model can count subgraphs of size $k$, and thereby overcomes a known limitation of low-order GNNs.
Second, we show how recursive pooling can exploit sparsity to reduce the computational complexity compared to the existing higher-order GNNs. 
More generally, we provide a (near) matching information-theoretic lower bound for counting subgraphs with graph representations that pool over representations of derived (sub-)graphs. We also discuss lower bounds on time complexity.
\end{abstract}

\section{Introduction}\label{intro}

Graph Neural Networks (GNNs) are powerful tools for  graph representation learning \citep{scarselli2008graph,kipf2017semi,hamilton2017inductive},
%They are first introduced in (\cite{scarselli2008graph}),  and Graph Convolutional Networks (GCNs)   (\cite{kipf2017semi}) and GraphSAGE (\cite{hamilton2017inductive}) are  popular instances of them.
and have been successfully applied to molecule property prediction, simulating physics, social network analysis, knowledge graphs, traffic prediction and many other domains \citep{duvenaud2015convolutional, defferrard2016convolutional, battaglia2016interaction, jin2018learning}. %\sj{some more refs?}
The perhaps most widely used class of GNNs, Message Passing Graph Neural Networks (MPNNs) \citep{gilmer2017neural,kipf2017semi, hamilton2017inductive,  xu2018powerful, scarselli2008graph}, follow an iterative message passing scheme to compute a graph representation. 
%  Many different message passing GNNs are proposed  and they have been empirically verified on several benchmark datasets, c.f., (\cite{gilmer2017neural, kipf2017semi, hamilton2017inductive,  xu2018powerful, scarselli2008graph}). Due to the simplicity of MPNNs, there has been a lot of interest to use them in practice for the representation learning on graphs.

Despite the empirical success of MPNNs,  their expressive power has been shown to be limited. For example, their discriminative power, at best, corresponds to the one-dimensional Weisfeiler-Leman (1-WL) graph isomorphism test 
\citep{xu2018powerful,morris2019weisfeiler}, so they cannot distinguish regular graphs, for instance. %(\cite{sato2020survey}). 
Likewise, they cannot count any induced subgraph with at least three vertices \citep{chen2020can}, or learn structural graph parameters such as clique information, diameter, conjoint or shortest cycle \citep{garg2020generalization}. 
Yet, in applications like computational chemistry, materials design or pharmacy \citep{elton2019deep,sun2020graph,jin2018learning}, the functions we aim to learn often depend on the presence or count of specific substructures, e.g., functional groups.

%such as those relying on representing molecules, 
%In contrast, in many applications, such as  molecular design and drug discovery (\cite{elton2019deep,sun2020graph}),  we  often require to learn functions which depend on the presence/count of specific substructures. 

The limitations of MPNNs result from their inability to distinguish individual nodes. To resolve this issue, two routes have been studied: (1) using unique node identifiers \citep{loukas2019graph,sato2019approximation,abboud2021surprising}, and (2) higher-order GNNs that act on $k$-tuples of nodes. Node IDs, if available, enable Turing completeness for sufficiently large MPNNs \citep{loukas2019graph}. Higher-order networks use an encoding of $k$-tuples and then apply message passing %, e.g., $k$-GNNs
\citep{morris2019weisfeiler}, or equivariant tensor operations \citep{maron2018invariant}. %\citep{maron2019universality}. 

The expressive power of MPNNs is often measured in terms of a hierarchy of graph isomorphism tests, specifically the $k$-Weisfeiler-Leman ($k$-WL) hierarchy. The $k$-order models in \citep{maron2018invariant} and \citep{maron2019provably} are equivalent to the $k$-WL and ($k+1$)-WL ``tests'', respectively, and are universal for the corresponding function classes \citep{azizian2020characterizing,maron2019universality,keriven2019universal}. Yet, these models are computationally expensive, operating on $\Theta(n^k)$ tuples and according to current upper bounds requiring up to $ O(n^k)$ iterations~\citep{kiefer2020iteration}.
%\sj{add iteration count?}
The necessary tradeoffs between expressive power and computational complexity are still an open question.
However, for specific classes of tasks this full universality may not be needed. Here, we study such an example 
of practical interest: counting substructures, %and characterize expressive power via this task, 
as proposed in \citep{chen2020can}.
In particular, we study if it is possible to count given substructures with a GNN whose complexity is between that of MPNNs and existing higher-order GNNs.

To this end, we study a generic scheme followed by many GNN architectures, including MPNNs and higher-order GNNs \citep{morris2019weisfeiler,chen2020can}: %\sj{more refs here?}
 select a collection of subgraphs of the input graph, encode these, and %(possibly iteratively) 
apply an aggregation function on this collection. First, we study the power of pooling \emph{by itself}, as a multi-set function over node features. We prove that $k$ \emph{recursive} applications on each node's neighborhood allow to count any substructure of size $k$. This is in contrast to \emph{iterative} MPNNs. We call this technique \emph{Recursive neighborhood pooling (RNP)}. While subgraph pooling relates to the graph reconstruction conjecture, our strategy has important differences. In particular, we show how the aggregation ``augments'' local encodings, if they play together and the subgraphs are selected appropriately, and this reasoning may be of interest for the design of other, even partially, expressive architectures. Moreover, our results show that the complexity is \emph{adjustable} to the counting task of interest and the sparsity of the graph. 

The strategy of pooling subgraph encodings has previously been used for counting in Local Relational Pooling (LRP) \citep{chen2020can}. LRP relies on an isomorphic encoding of subgraphs, which is expensive -- e.g., the relational pooling it uses requires $O(k!)$ time for a subgraph of size $k$. Other higher-order GNNs would be expensive, too, as high orders are needed for complete isomorphism power. A major difference to our RNP is that our recursion uses subgraphs of \emph{varying} sizes and structures, many of them much smaller -- adapted to the graph structure and specific counting task. %, and can thereby benefit from sparsity, for instance.

Furthermore, we study lower bounds on GNNs that count motifs. We show an information-theoretic lower bound on the number of subgraphs to encode, as a function of an encoding complexity. We also transfer computational lower bounds that apply to any counting GNN. The lower bounds show that the recursive pooling is close to tight.

In short, in this paper, we make the following contributions:
\vspace{-5pt}
\begin{itemize}[leftmargin=8pt] \setlength{\itemsep}{0pt}
\item We study the power of pooling encodings of subgraphs, and show that pooling, as an injective multi-set function, is sufficient \emph{by itself} for counting when applied \emph{recursively} on appropriate subgraphs, remarkably without relying on other encoding techniques or node IDs. This is different from any other strategy we are aware of in the literature.

\item We analyze the complexity of recursive pooling, as a function of the task and input graph.

\item We provide complexity lower bounds for pooling and general GNN architectures that count motifs. For instance, we show a lower bound on the number of subgraphs that need to be encoded.

\iffalse 

\item We introduce Recursive Neighborhood Pooling Graph Neural Networks (RNP-GNNs), a  flexible class of higher-order graph neural networks, that provably allow to design graph representation networks with any expressive power of interest, in terms of counting (induced) substructures. 

%\item More interestingly, it is shown that appropriate hyper-parameters allow to achieve RNP-GNNs which are at least as powerful as $k-$GNNs, for $k\ge2$. 

\item We show that RNP-GNNs offer computational gains over existing models that count substructures: an exponential improvement in terms of the ``tolerable'' size of the encoded neighborhoods compared to LRP networks, and much less complexity in sparse graphs compared to $k-$GNN and $k-$IGN.

\item We provide an information-theoretic lower bound on the complexity of a general class of GNN that can count (induced) substructures with at most $k$ vertices.

%\item We show that any graph representation learning which encodes input graphs to a number of (much smaller) graphs to use them in counting substructures, such as LRP networks, indeed requires $n^{\Omega(1)}$ encoded graphs. In that sense, no exponential improvement can be achieved even with so complicated representation functions.
\fi
 
\end{itemize}

%This paper is organized as follows. An extensive literature review   on the recent advances in the field of representation learning on graphs and preliminary knowledge for undestainding this paper is provided in the appendix, due to the space constraint (see Section  \ref{rw}).  Section \ref{ps} is devoted to the formal presentation of the problem and results, and in Section \ref{conclusion} we conclude the paper. Proofs are also deferred to the appendices. 

\section{Background}\label{bg}

\paragraph{Message Passing Graph Neural Networks.} 
%Message passing graph neural networks (MPNNs) follow an iterative message passing scheme  to compute the representation of each node in a graph. 
Let $G= (\mathcal{V}, \calE, X)$ be an attributed graph with $|\mathcal{V}| = n$ nodes. %, where $[n]:=\{1,2,\ldots,n\}$. 
Here,  $X_v \in \calX$ denotes the initial attribute   of $v \in \calV$, where $\calX \subseteq \N$ is a (countable) domain.%\footnote{ Without loss of generality, we  assume that $\calX \subseteq \N$. We may assume that the labels are one-hot encoded.}

A typical Message Passing Graph Neural Network (MPNN) first computes a representation of each node, and then aggregates the node representations via a readout function into a representation of the entire graph $G$. 
The representation $h^{(i)}_v$ of each node $v \in \calV$ is computed iteratively by aggregating the representations $h^{(i-1)}_u$ of the neighboring vertices $u$:
%  To find the representation of $G$, a GNN proceeds as follows. First let $h_v^{(0)} = X_v$ as the first candidate for the representation of each node, for all $v \in \calV$. Then, update nodes' representations as 
\begin{align}
m_v^{(i)} &= \textsc{Aggregate}^{(i)} \Big(\lb h_u^{(i-1)} : u \in \calN(v)\rb  \Big), \quad
h_v^{(i)} =\textsc{Combine}^{(i)} \Big(h_v^{(i-1)} ,   m_v^{(i)} \Big),
\end{align}
for any $v \in \calV$, for $k$ iterations, and with $h_v^{(0)} = X_v$. The \textsc{Aggregate}/\textsc{Combine} functions are parametrized, and   $\lb . \rb $  denotes a multi-set, i.e., a set with (possibly) repeating elements.  A graph-level representation can be computed as  
%\begin{align}
$h_G = \textsc{Readout}\big(\lb h_v^{(k)} : v \in \calV \rb\big)$,
%\end{align}
where \textsc{Readout} is a learnable aggregation function. For representational power, it is important that the learnable functions are injective %, which can be achieved, e.g., if the AGGREGATE function is a summation and  COMBINE is a weighted sum concatenated with an MLP 
\citep{xu2018powerful}. %READOUT is also assumed to be a  concatenation of all the representations of the nodes.
%\footnote{There has been  also considered jumping knowledge to fully exploit all the information provided to the iterations. }

\paragraph{Higher-Order GNNs.}  
To increase the representational power of GNNs, several higher-order GNNs have been proposed. 
In \emph{$k$-GNN}, message passing is applied to $k-$tuples of nodes, inspired by $k$-WL \citep{morris2019weisfeiler}. At initialization, each $k$-tuple is labeled such that two $k$-tuples are labeled differently if their induced subgraphs are not isomorphic. As a result, $k$-GNNs can count (induced) substructures with at most $k$ vertices even at initialization. 
Another class of higher-order networks applies (linear) equivariant operations, interleaved with coordinate-wise nonlinearities, to order-$k$ tensors consisting of the adjacency matrix and input node attributes \citep{maron2018invariant,maron2019provably,maron2019universality}. These GNNs are at least as powerful as $k-$GNNs, and hence they too can count substructures with at most $k$ vertices. All these methods need $\Omega(n^k)$ operations. 
%
%are \emph{$k-$IGNs}, which are constructed with linear invariant/equivariant feed-forward layers, whose inputs consider  graphs via adjacency matrices \citep{maron2019universality}.   $k-$IGNs are at least as powerful as $k-$GNNs, and hence they too can count substructures with at most $k$ vertices.   However, both methods need $O(n^k)$ operations. 
%
%\subsection{Local Relational Pooling}  
%
\emph{Local Relational Pooling (LRP)} \citep{chen2020can} was designed for counting and applies relational pooling \citep{murphy2019relational, murphy2019janossy} on local neighborhoods, i.e., one pools over evaluations of a permutation-sensitive function applied to all $k!$ permutations of the nodes in a $k$-size neighborhood of each node.

%Specifically for counting substructures, \cite{chen2020can} propose \emph{Local Relational Pooling (LRP)} networks. LRPs apply Relational Pooling (RP) networks \citep{murphy2019relational, murphy2019janossy} on the neighborhoods around each vertex. RP networks  use permutation-variant functions and convert them to a  permutation-invariant function by summing over all permutations. This summation is computationally expensive. 
%as they use summation over all the permutations of a given size, their model is intractable for even small local neighborhoods and approximations are required (\cite{chen2020can}). 

\section{Other Related Works}\label{rw}

\textbf{Expressive power.}
Several other works have studied the expressive power of GNNs as function approximators \citep{azizian2020characterizing}. \cite{scarselli2009} extend universal approximation from feedforward networks to MPNNs, using the notion of \emph{unfolding equivalence}, i.e., functions on computation trees. Indeed, graph distinction and function approximation are closely related \citep{chen2019equivalence,azizian2020characterizing,keriven2019universal}. % establish an equivalence between the graph isomorphism problem and the power to approximate permutation invariant functions on graphs. 
\cite{maron2019universality} and \cite{keriven2019universal} show that higher-order, tensor-based GNNs  provably achieve universal approximation of permutation-invariant functions on graphs, and
 %It is worth mentioning that there is an equivalence between the expressive power in the graph isomorphism problem and the function approximation power in graph neural networks (\cite{chen2019equivalence}). 
 \cite{loukas2019graph} analyzes expressive power under depth and width restrictions. Studying GNNs from the perspective of local algorithms, \cite{sato2019approximation} show that GNNs can approximate solutions to certain combinatorial optimization problems.
% The approximation guarantee and expressive power of many classes of currently propose GNNs are  derived; see  (\cite{azizian2020characterizing}).  Also, under %certain criteria, it is also shown that GNNs  are able to approximate solutions for particular combinatorial optimization problems (\cite{sato2019approximation}). 

%%%%%GK%%%%%%%%
%\textbf{Graph Kernels (GKs). }
%Graph Kernels (GKs) are proposed as a  variant of kernels defined on graphs  \cite{vishwanathan2010graph, shervashidze2009efficient, shervashidze2011weisfeiler}. 
%In a GK, we extract the information we require in learning from the graph, and use that to make the graph representation. Instances of these informations include but not limited to sub-trees, shortest paths, random walks, etc. Rational kernels, R-convolution kernels are well-known kernels which are widely used in practice  \cite{vishwanathan2010graph}. The combine of a base kernel method with a $1-$WL algorithm results in a so called WL-GK, which is a boosted version of the initial kernel. Graph kernels based on counting small sub-structures are also proposed; however, since counting is expensive, approximation algorithms are required \cite{shervashidze2009efficient}. See \cite{kriege2020survey, nikolentzos2019graph} for a comprehensive survey on GKs. 
%A GNN based on using the Neural Tangent Kernels (NTK), called Graph Neural Tangent Kernel (GNTK) is also proposed \cite{du2019graph} to combine ideas from GK and GNNs to each other.

%%%%%%%%%%%%%%%

%%%%%%%reconstruction%%%%%%

%%%%%%subgraphs%%%%%%%%%%%%

\textbf{Subgraphs and GNNs. }
Having infromation about subgraphs can be quite helpful in various  graph representation algorithms \citep{liu2019n, monti2018motifnet, liu2020neural, yu2020sumgnn, meng2018subgraph, cotta2020unsupervised, alsentzer2020subgraph, huang2020graph}.
%For example, in link prediction,   \cite{zhang2018link} combine  local information with GNNs.  
%The idea of strengthening  GNNs via local neighborhoods compression is  defined  in \citep{li2019hierarchy}, but no specific way is introduced in that work to how perform such an aggregation on local neighborhoods, unless some known results on based on paths/random walks.  
%A novel method based on combining GNNs and a clustering algorithm is proposed in (\cite{ying2018hierarchical}). 
For example, for graph comparison (i.e., testing whether a given (possibly large) subgraph exists in the given model), \cite{rex2020neural} compare the outputs of GNNs for small subgraphs of the two graphs.
%
%it is also suggested to compare the GNN outputs constructed based on their embedded subgraphs with small radii (\cite{rex2020neural}).
 To improve the expressive power of GNNs, \cite{bouritsas2020improving} use features that are counts of specific subgraphs of interest.
%
% it is suggested to count specific subgraphs of interest and consider them as features, which results in a  kernel-based representation (\cite{bouritsas2020improving}). 
 %; however, this kernel-based method simply fails in practice since we need extensive search even for  small subgraphs. 
Another example is   \citep{vignac2020building}, where   an  MPNN is strengthened by learning local context matrices around vertices.
Recent works have also developed GNNs that pass messages on ego-nets \citep{you2021identity, sandfelder2021ego}.
With motivation from the reconstruction conjecture, \cite{cotta2021reconstruction} process node-deleted subgraphs with individual MPNNs, and then pool them with a DeepSets model to get a representation of the original graph.

 \section{ \ggnn }\label{ps}
 
 Let $ G= (\calV, \calE, X)$ be an attributed input graph with $|\calV| = n$ nodes, and let $h_v^{(0)} = X_v$ be the initial representation of each node $v$.
 %
 %\subsection{Recursive Neighborhood Pooling} 
%
In this work, we study architectures that first find representations of a collection of $m$ subgraphs $G_i$ and then aggregate (pool) over   these representations with a multi-set function, i.e.,
\begin{equation}
    \textsc{Aggregate}( \lb \psi(G_i): i \in [m] \rb),\quad [m] =\{1, \ldots, m\}.
\end{equation}
It is clear that if the $\psi$ count a subgraph structure $H$, then the entire model can count $H$. 
In particular, we aim to apply this strategy to obtain the representations $\psi$, too. To appreciate the challenges in doing so, recall two such examples. First, MPNNs follow this pooling strategy, by iteratively aggregating over node neighborhoods, and then aggregating all node representations into a graph representation. However, it is known that MPNNs can count at most star structures or edges. This is because they represent a local computation tree, which loses structural information about node identities and connectivity. Second, this strategy is at the heart of the Graph Reconstruction Conjecture \citep{kelly1957congruence}, which conjectures that a graph $G$ can be reconstructed from its subgraphs $\lb G_v = G \setminus \{v\} : v \in V(G)\rb$ (Appendix~\ref{sec:reconstruction}). This, however, is unknown for general graphs. Although the $G_v$ retain some structure, we lose information about their \emph{structural relationship}.
In summary, encoding structural information is the key question. 

%\begin{figure}[t]
%\centering
%\includegraphics[scale=.4]{01.pdf}
%\caption{Illustration of a Recursive Neighborhood Pooling GNN (RNP-GNN) with recursion parameters $(2,2,1)$.
% To compute the representation of  vertex $v$ in the given input graph (depicted in the top left of the figure), we first recurse on $G(\calN_2(v)\setminus \{v\})$ (top right of figure).   To do so, we find the representation of each vertex $u \in G(\calN_2(v)\setminus \{v\})$. 
%   For instance, to compute the representation of $u_1$, we  apply an RNP-GNN with recursion parameters  $(2,1)$ and aggregate  $G((\calN_2(v)\setminus \{v\}) \cap ( \calN_2(u_1) \setminus \{u_1\}))$, which is shown in the bottom left of the figure. To do so, we recursively apply an RNP-GNN  with recursion parameter $(1)$ on $G((\calN_2(v)\setminus \{v\}) \cap (\calN_2(u_1) \setminus \{u_1\} ) \cap (\calN_{1}(u_{11})\setminus \{u_{11}\}))$, in the bottom right of the figure.\sj{Put caption next to the figure with sidecap} }
%\label{fig1}
%\end{figure}

\begin{SCfigure}[][ht]
\centering
\includegraphics[scale=.4]{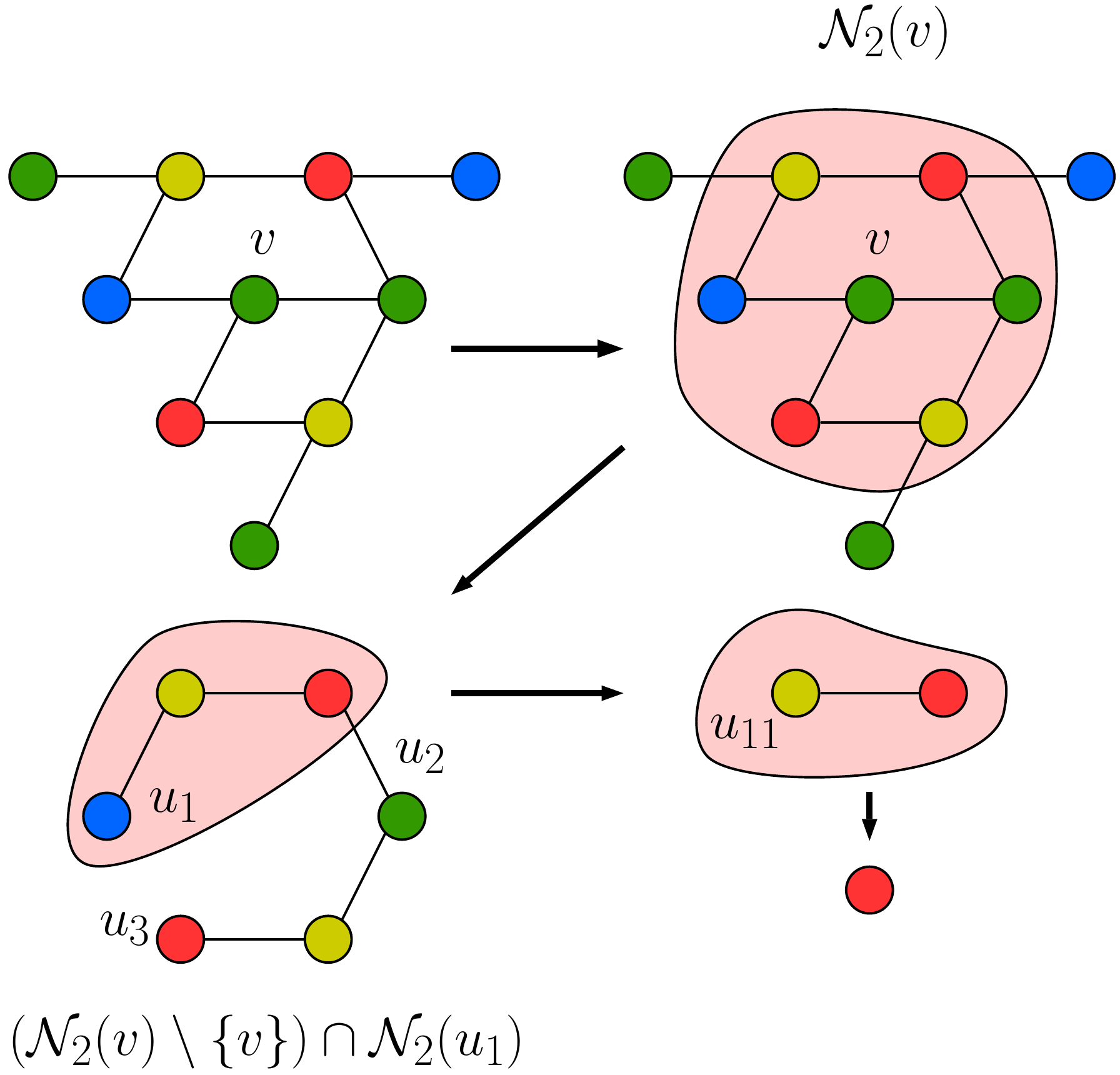}
\caption{Illustration of a Recursive Neighborhood Pooling GNN (RNP-GNN) with recursion parameters $(2,2,1)$.
 To compute the representation of  node $v$ in the given input graph (depicted in the top left of the figure), we first recurse on $G(\calN_2(v)\setminus \{v\})$ (top right of figure).   To do so, we find the representation of each node $u \in G(\calN_2(v)\setminus \{v\})$. 
   For instance, to compute the representation of $u_1$, we  apply an RNP-GNN with recursion parameters  $(2,1)$ and aggregate  $G((\calN_2(v)\setminus \{v\}) \cap ( \calN_2(u_1) \setminus \{u_1\}))$, which is shown in the bottom left of the figure. To do so, we recursively apply an RNP-GNN  with recursion parameter $(1)$ on $G((\calN_2(v)\setminus \{v\}) \cap (\calN_2(u_1) \setminus \{u_1\} ) \cap (\calN_{1}(u_{11})\setminus \{u_{11}\}))$, in the bottom right of the figure.
   }
\label{fig1}
\end{SCfigure}

Hence, to represent the counting function $\psi$ over potentially large neighborhoods, an MPNN does not suffice. But, aggregation over input node attributes, along with edge information, can count edge types, i.e., tiny subgraphs. Hence, we \emph{recursively} apply aggregation on smaller sub-neighborhoods while remembering structural information, with node-wise aggregation as the base case.
For intuition, suppose we aim to count the occurrence of subgraph $H$ in the $r_1$-neighborhood $\calN_{r_1}(v)$ of a node $v$. To do so, we may count $H_v = H \setminus \{v\}$ in the smaller graph $\calN_{r_1}(v) \setminus \{v\}$. But, to combine these counts with the presence of $v$ to ``complete'' $H$, we must know how the $H_v$ are connected to $v$ in the screened graph. Hence, to retain structure information, we mark the neighbors of $v$. % by augmenting their node features.
%
%To retain it, we augment the node features $h^{t-1}_u$ of $u \in \calN_{r_1}(v) \setminus \{v\}$ as $h^{t-1}_{u,\text{aug}} = (h^{t-1}_u, \ind[ (u,v) \in \calE])$. 
We then recursively call neighborhood pooling to process smaller neighborhoods $\calN_{r_2}(u)$ \emph{within} $\calN_{r_1}(v) \setminus \{v\}$. This could, e.g., learn to count marked versions of $H_v$.
The radii $r$ of neighborhoods may differ in recursive calls. In Section~\ref{sec:expressive}, we relate their size to $H$.

%In summary, 
\emph{Recursive neighborhood pooling} $\textsc{RNP-GNN}(G,\{h^{\text{in}}_u\}_{u \in \calV(G)}, (r_1, \ldots, r_{\tau}))$ takes a graph with node features and a sequence of neighborhood radii for different recursions, and returns a set of node encodings $\{h_v\}_{v \in \calV(G)}$. For any $v \in G$, RNP-GNN first constructs $v$'s neighborhood, removes $v$ and marks its neighbors:
\begin{align}
    G_v &\leftarrow \calN_{r_1}(v)\setminus \{v\}, \quad\quad h^{\text{in}}_{u,\text{aug}} = (h^{\text{in}}_u, \ind[ (u,v) \in \calE(G_v)]).
\end{align}
Then we aggregate over subgraph representations.
If $\tau = 1$ (base case), we use the input features:
\begin{equation}
    h_v \leftarrow \textsc{Aggregate}^{(\tau)}(h^{\text{in}}_v, \lb h^{\text{in}}_{u,\text{aug}} : u \in G_v \rb).
\end{equation}
If $\tau > 1$, we recursively represent neighborhoods of nodes in $G_v$:
\begin{align}
    \{\hat{h}_{v,u}\}_{u \in G'} &\leftarrow \textsc{RNP-GNN}\big(G_v, \{h^{\text{in}}_{u,\text{aug}} \}_{u \in G_v}, (r_2,r_3,\ldots,r_{\tau})\big)\\
    h_v &\gets \textsc{Aggregate}^{(\tau)} \Big ( h_v^{\text{in}}, \lb \hat{h}_{u,v}: u \in  G_v\rb \Big)
\end{align}
For aggregation, we can use, e.g., the injective multi-set function from \citep{xu2018powerful}:
\begin{equation}
    \textsc{Aggregate}^{(\tau)} \Big ( h_v, \{h_u\}_{u \in G_v}\Big) = \textsc{MLP}^{(\tau)} \Big ( (1+\epsilon)h_v+\sum\nolimits_{ u \in  G_v  } \hat{h}_{u}\Big).
\end{equation}
The final readout aggregates over the final node representations of the entire graph. Figure \ref{fig1} illustrates an RNP-GNN with recursion parameters $(2,2,1)$, and Appendix~\ref{psc} provides pseudocode.

While MPNNs also encode a representation of a local neighborhood, the recursive representations differ as they take into account \emph{intersections} of neighborhoods. As a result, as we will see in Section~\ref{sec:expressive}, they retain more structural information and are more expressive than MPNNs. Models like $k$-GNN and LRP also compute encodings of subgraphs, and then update the resulting representations via message passing. We can do the same with the neighborhood representations computed by RNP-GNNs to encode more global information, although our representation results in Section~\ref{sec:expressive} hold even without that. %In Section~\ref{sec:complexity}, we compare the computational complexity of RNP-GNN and these other models.

\section{Expressive Power of Recursive Pooling}\label{sec:expressive}

%In this part, we provide theoretical guarantee for learning via subgraphs summarization networks. Two results are provided regarding the expressive power and the complexity of the subgraphs summarization networks. 

In this section, we analyze the expressive power of RNP-GNNs.

\begin{figure}[t]
\centering
\includegraphics[scale = 0.3]{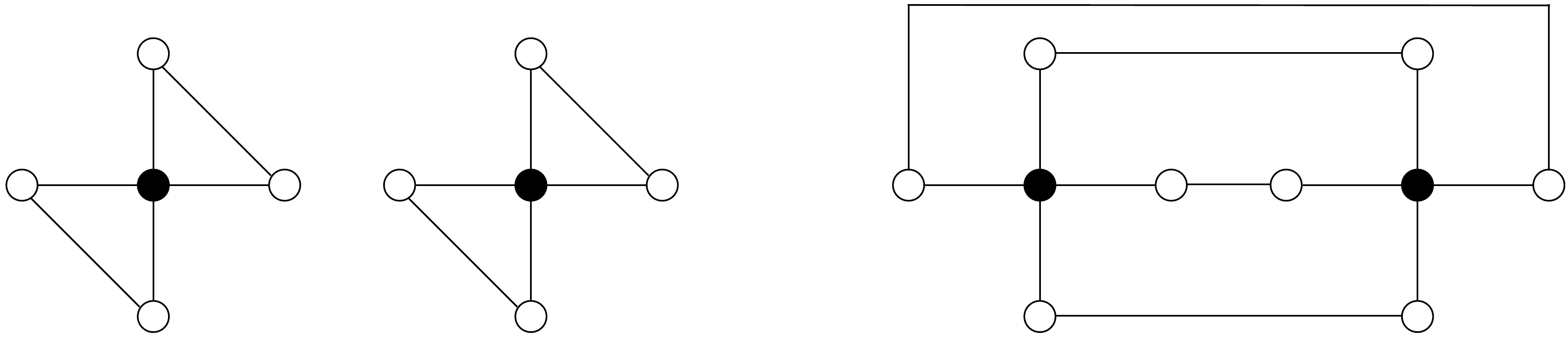}
\caption{MPNNs cannot count substructures with three nodes or more \citep{chen2020can}. For example, the graph with  black center vertex on the left cannot be counted, since the two graphs on the left result in the same node representations as the graph on the right. \label{ref:count}}
\end{figure}

\subsection{Counting (Induced) Substructures}

In contrast to MPNNs, which in general cannot count substructures of three vertices or more \citep{chen2020can}, in this section we prove that for any set of substructures, there is an RNP-GNN that provably counts them.
We  begin with a few definitions. 

%Recall that GNNs are not able to count any substructure with at least three vertices; see Figure 2.
%In contrast, in this section we prove that for any set of substructures, there is an RNP-GNN which provably counts them.
%We  begin with a few definitions. 

\begin{definition}
Let $G,H$ be arbitrary, potentially attributed simple  graphs, where $\calV$ is the set of nodes in $G$. Also, for any $\calS \subseteq \calV$, let $G(\calS)$ denote the subgraph of $G$ induced by $\calS$. The \emph{induced subgraph count function} is defined as
\begin{align}
C(G;H) := \sum\nolimits_{\calS \subseteq \calV} \ind \{G(\calS)\cong H\},
\end{align}
i.e., the number of subgraphs of $G$ isomorphic to $H$. %For unlabeled $H$, the function is defined analogously. 
\end{definition}
To relate the size of encoded neighborhoods to the substructure $H$, we will need a notion of \emph{covering sequences} for graphs. %Our definition uses distances on graphs.
\begin{definition}
Let   $H=(\calV_H,\calE_H)$ be a simple connected graph. For any $\calS \subseteq \calV_H$ and $v \in \calV_H$,  define the covering distance of $v$ from $\calS$ as  
\begin{align}
\bar{d}_H(v;\calS) := \max_{u \in \calS} d(u,v),
\end{align}
where $d(.,.)$ is the shortest-path distance in $H$.
\end{definition}

\begin{definition}
Let $H$ be a simple connected graph on $\tau+1$ vertices. A permutation of vertices, such as $(v_1,v_2,\ldots,v_{\tau+1})$,  is called a \emph{vertex covering  sequence} with respect to a sequence $\mathbf{r}  = (r_1,r_2,\ldots,r_\tau) \in \N^\tau$  called a \emph{covering sequence}  if and only if 
\begin{align}
\bar{d}_{H'_i}(v_i;\calS_i) \le r_i,\label{eq::ref::alg2}
\end{align}
for any $i \in [\tau+1]=\{1,2,\ldots,\tau+1\}$, where $\calS_i = \{v_{i},v_{i+1}\ldots,v_{\tau+1}\}$ and $H'_i = H(\calS_{i})$ is the subgraph of $H$ induced by the set of vertices $\calS_{i}$. We also say that $H$ \emph{admits} the
covering sequence $\mathbf{r}  = (r_1,r_2,\ldots,r_{\tau}) \in \N^\tau$ if there is a vertex covering sequence for $H$ with respect to $\mathbf{r}$.
 %Let  $\calC_H({\bf r})$  denote the set of all vertex covering  sequences with respect to $\mathbf{r}$ for $H$.
\end{definition}
In particular, in a covering sequence we first consider the whole graph as a local neighborhood of one of its nodes with radius $r_1$. Then, we remove that node and compute the covering sequence of the remaining graph. Figure \ref{fig:cov} shows an example of  covering sequence computation. An important property, which holds by definition, is that if $\bold{r}$ is a covering sequence for $H$, then any $\bold{r'} \ge \bold{r}$ (coordinate-wise) is also a covering sequence for $H$. 
%To find a covering sequence for a given graph $H$, we can simply choose a spanning tree of it (since it is connected it will have at least one spanning tree), and then,  choose one of the leaves tree as the first element in the vertex covering sequence. Removing that vertex will not make the resulting graph not-connected. We then proceed according to the leaves of the resulting trees after removing the vertices and we finally obtain a vertex covering sequence and consequently a covering sequence. Similar argument also shows that $(k-1,k-2,\ldots,1)$ is a covering sequence for any connected graph with $k$ vertices.  However, in this paper since we only need to compute the covering  sequence for small subgraphs ($k=O(1)$), we may simply find a covering sequence by search.   

Note that any connected graph on $k$ nodes admits at least one covering sequence, which is $(k-1,k-2,\ldots,1)$. To observe this fact, note that  in a connected graph, there is at least one node that can be removed and the remaining graph still remains connected. Therefore, we may take this node as the first element of a vertex covering sequence, and inductively find the other  elements. Since the diameter of a connected graph  with $k$ vertices is always bounded by $k-1$, we achieve the desired result. 
However, we will see in the next section that,  when using covering sequences to identify sufficiently powerful RNP-GNNs, it is desirable to have covering sequences with low $r_1$, since the complexity of the resulting RNP-GNN depends on $r_1$. %In Appendix~D, we provide an algorithm to find such covering sequences in polynomial time.

More generally,  if $H_1$ and $H_2$ are (possibly attributed) simple graphs on $k$ nodes  and $H_1 \Subset H_2 $, i.e., $H_1$ is a subgraph of $H_2$  (not necessarily induced subgraph), then it follows from the definition that any covering sequence for $H_1$ is also a  covering sequence for $H_2$. 
As a side remark, as illustrated in Figure~\ref{fig4}, covering sequences need not always to be decreasing.
\begin{figure}[t]
\centering
\includegraphics[scale = 0.35]{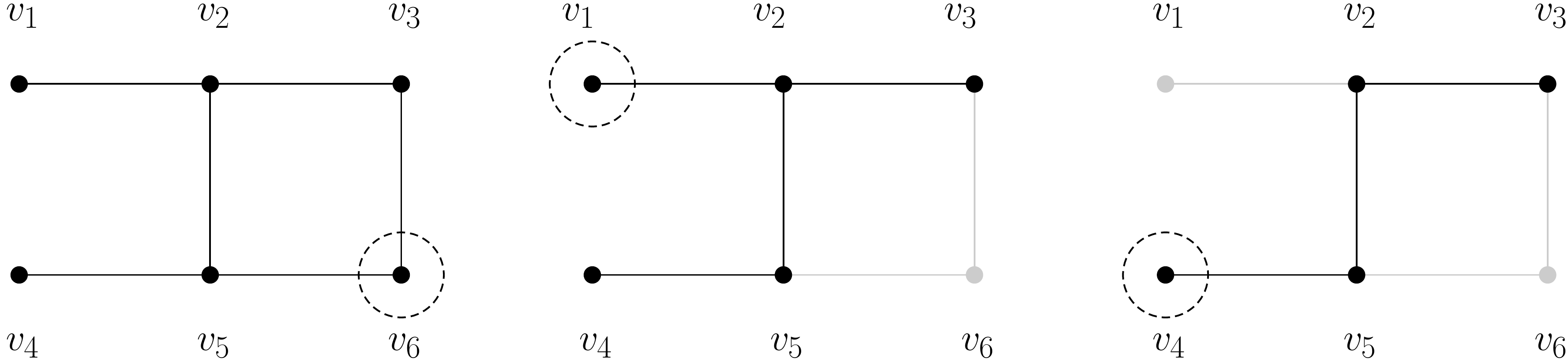}
\caption{Example of a covering sequence  computed for  the  graph on the left. For this graph, $(v_6,v_1, v_4, v_5,v_3,v_2)$ is a vertex covering sequence with respect to the covering sequence $(3, 3, 3, 2,1)$. The first two computations to obtain this covering sequence are depicted in the middle and on the right. \label{fig:cov} }
\end{figure}
Using covering sequences, we can show the following result.

\begin{theorem}\label{th1}
Consider a set of (possibly attributed) graphs  $\calH$  on  $\tau+1$ vertices, such that  any $H \in \calH$  admits the covering sequence  $(r_1,r_2,\ldots,r_\tau)$. Then, there is an RNP-GNN $f( \cdot ;\theta)$ with  recursion parameters   $(r_1,r_2,\ldots,r_\tau)$ that can count any $H \in \calH$. In other words,    for any $H \in \calH$, if  $C(G_1;H)\neq C(G_2;H)$, then  $f(G_1;\theta) \neq f(G_2;\theta)$. 
The same  result also holds for the non-induced subgraph count function. 
\end{theorem}

%This theorem tells us that with a subgraphs summarization network, we are able to learn the number of subgraph of each $H$, such that there is a covering sequence for $H$, like $(s_1,s_2,\ldots,s_{t'})$, where  $t'\le t$ and $s_i\le r_i$ for all $i\in[t']$. 
Theorem \ref{th1} states that, with appropriate recursion parameters, any set of (possibly attributed) substructures can be counted   by an RNP-GNN. Interestingly,  induced and non-induced subgraphs can be both counted in RNP-GNNs\footnote{For simplicity, we assume that $\calH$ only contains   $\tau+1$ node graphs. If $\calH$ includes graphs with strictly less than $\tau+1$ vertices, we can simply append a sufficient number of zeros to their covering sequences.}. We prove Theorem~\ref{th1} in Appendix~\ref{app:proof-counting}. The main idea is to show that we can implement the intuition for recursive pooling outlined in Section~\ref{ps} formally with the proposed architecture and multiset functions.

The theorem holds for any covering sequence that is valid for all graphs in $\calH$. For any graph, one can compute a covering sequence by computing a spanning tree, and sequentially pruning the leaves of the tree. The resulting sequence of nodes is a vertex covering sequence, and the corresponding covering sequence can be obtained from the tree too (Appendix~D). A valid covering sequence for all the graphs in $\calH$ is the coordinate-wise maximum of all these sequences.

For large substructures, the sequence $(r_1,r_2,\ldots,r_\tau)$ can be long or include large numbers, and this will affect the computational complexity of RNP-GNNs. For small, e.g., constant-size substructures, the recursion parameters are also small (i.e., $r_i = O(1)$ for all $i$), raising the hope to count these structures efficiently. In particular, $r_1$ is an important parameter. In Section~\ref{sec:complexity}, we analyze the complexity of RNP-GNNs in more detail.

%However, the key point is that the sequence $(r_1,r_2,\ldots,r_t)$ in some cases is long or includes large numbers, which may lead to intractable computational complexity in the model. Hopefully, for small graphs ($k=O(1)$ vertices), the recursion parameters are also small (i.e., $r_i = O(1)$ for all $i$)  and this means that we may hope to count them using an efficient way. Also, since we consider neighborhoods recursively (i.e., neighborhoods within a neighborhood), the length of the sequence $(r_1,r_2,\ldots,r_t)$ is less important in sparse graphs. In contrast,  the first elements of the sequence such as $r_1,r_2,\ldots$ are more important because they may result to large neighborhoods. We provide more explanations about this issue when we analyze the complexity of the model. 

\subsection{A Universal Approximation Result for Local Functions}

Theorem \ref{th1} shows that RNP-GNNs can count substructures if their recursion parameters are chosen carefully. Next, we provide a universal approximation result, which shows  that   they can represent any function related to local neighborhoods  or small subgraphs in a graph.

First, we  recall that for a graph $G$,  $G(\calS)$ denotes the subgraph of $G$  induced by the set of vertices $\calS$.

\begin{definition}
A function $\ell :\G_n \to \R^d$ is called an $r-$local graph function if 
\begin{align}
\ell(G) = \phi (\lb \psi(G(S)) : \calS \subseteq \calV, |\calS| \le r \rb ),
\end{align}
where  $\psi:\G_r \to \R^{d'}$ is a function on graphs  and   $\phi$ is a multi-set function.
%, such that $\psi(G) \neq 0$ if and only if $G$ is connected. 
%Also, a function $\ell:\calV \times \G_n \to \R^d$ is called an   $r-$local  vertex function if $\ell(v;G) = \psi(G(\calN_r(v)))$ for all $v \in \calV=[n]$, and for a function $\psi: \G_r \to \R^d$. 
\end{definition}

In other words, a local function only depends on  small substructures. %Now we prove the following result.

\begin{theorem}\label{th2}
For any  $r-$local graph function $\ell(.)$, there exists an RNP-GNN $f(.;\theta)$  with  recursion parameters $(r-1,r-2,\ldots,1)$ such that $f(G;\theta) = \ell(G)$ for any $G \in \G_n$. 
\end{theorem}

As a result, we can provably learn all the local information  in a graph with  an appropriate RNP-GNN.  Note that we still need recursions, because the function $\psi(.)$ may be an arbitrarily difficult graph function.  However, to achieve the full generality of such a universal approximation result, we need to consider large recursion parameters ($r_1=r-1$) and injective aggregations in the RNP-GNN network. For universal approximation, we may also need high dimensions if fully connected layers are used  for aggregation (see the proof in Appendix~\ref{app:proof-universal} for more details). 

As a remark, for $r=n$, achieving universal approximation on graphs implies solving the graph isomorphism problem. But, in this extreme case, the computational complexity of RNP is
%that in this case ($r_1 = \Omega(n)$) the complexity of the model 
in general not polynomial in $n$.

\subsection{Computational Complexity}
\label{sec:complexity}

\setlength\intextsep{5pt}
\begin{wrapfigure}{r}{0.4\textwidth}
  \begin{center}
    \includegraphics[width=0.36\textwidth]{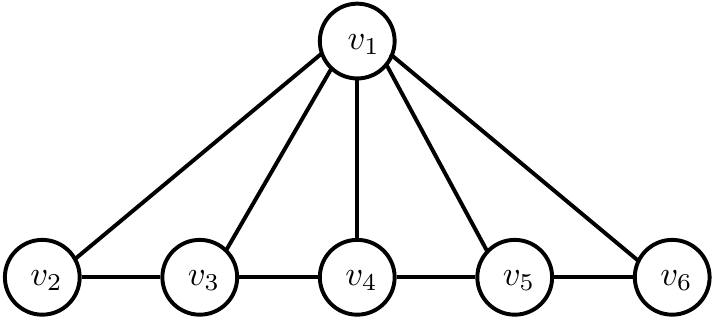}
  \end{center}
  \caption{For the above graph,  $(v_1,v_2,\ldots,v_6)$ is a vertex covering sequence. The corresponding covering sequence $(1,4,3,2,1)$ is not decreasing.}
  \label{fig4}
\end{wrapfigure}
\setlength\intextsep{\oldintextsep}

The computational complexity of RNP-GNNs is graph-dependent. For instance, we need to compute the set of local neighborhoods, which is cheaper for sparse graphs. Moreover, in the recursions, we use intersections of neighborhoods which become smaller and sparser.
%One important aspect of the  proposed network is that the complexity of it is graph-dependent. We need  pre-processing on each graph to obtain the set of its local neighborhoods, and it can be difficult if the graph is not sparse. 
\iffalse
A complexity measure existing in the literature is the tensor order. For higher-order networks, e.g.,  $k-$IGNs,  we need to consider tensors  in $\R^{n^k}$. The space complexity is then $O(n^k)$ and the time complexity can be even more, dependent on the algorithm used to process tensors. In general, for a message passing algorithm on  graphs,  the complexity of the model depends linearly on the number of vertices (if the graph is sparse). Therefore, to bound the complexity of a method, we need to bound the number of node representation updates, which we do in the following theorem.
\fi
\begin{theorem}\label{th3}
Let $f(.;\theta):\G_n \to \R^d$ be an RNP-GNN with recursion parameters $(r_1,r_2,\ldots,r_\tau)$. Assume that the observed graphs $G_1,G_2,\ldots$,  whose representations we compute, satisfy the following property:
$\max_{v \in [n]} |\calN_{r_1}(v)| \le c,$ for a constant $c$. Then the number of node updates in the RNP-GNN  is $O(nc^\tau)$. 
\end{theorem}
In other words, if $c = n^{o(1)}$ and $\tau=O(1)$, 
then  RNP-GNN requires relatively few updates (that is, $n^{1+o(1)}$). If the maximum degree of the given graphs is $\Delta$, then  $c = O(r_1\Delta^{r_1})$. Therefore, similarly, if $\Delta = n^{o(1)}$ then we can count with at most $n^{1+o(1)}$ updates. 
Additional gains may arise from rapidly shrinking neighborhoods, which are not yet accounted for in Theorem~\ref{th3}. To put this in context, the higher-order GNNs based on tensors or $k$-WL would operate on tensors of order $n^{\tau+1}$.

\begin{wraptable}{r}{0.6\textwidth}
    \parbox{0.6\textwidth}{
    \centering
    \vspace{-10pt}
    \caption{\small Time complexity of various models. $\Delta$ is the max-degree, and '$-$' means the complexity  is not  polynomial in $n$. }
    \label{tab:exp2}
    \scalebox{0.8}{
    \begin{tabular}{lcccc}
    \toprule
         Model & worst-case &  $\Delta = n^{o(1)}$ & $\Delta = O(\log(n))$ &  $\Delta = O(1)$ \\
         \midrule
         LRP & $-$ &$-$ & $-$ & $O(n)$ \\
         $k-$WL & $ n^k$ & $ n^k$ & $n^k$ & $n^k$\\
         RNP & $n^k$ & $n^{1+o(1)}$ & $\tilde{O}(n)$ & $O(n)$\\ 
         \bottomrule
    \end{tabular}
    }
    }
\end{wraptable}

\iffalse
, compared to the higher-order networks ($O(n^{t+1})$).\sj{strictly, this is $n^{1+t \cdot o(1)}$ vs $O(n^{t+1})$} Also, in this case, finding neighborhoods is not difficult, since neighborhoods are small ($n^{o(1)}$).
Note that if the maximum degree of the given graphs is $\Delta$, then  $c = O(r_1\Delta^{r_1})$. Therefore, similarly, if $\Delta = n^{o(1)}$ then we can count with at most $n^{1+o(1)}$ updates. 
%Then, the complexity of the subgraphs summarization network (up to constants not depending on $n$) is 
%\begin{align}
%nc_n^t \le n(r_1+1)\Delta^{tr_1} = n^{1+(tr_1\log(\Delta)+ \log(r_1+1))/\log(n) }.
%\end{align}

%\begin{corollary}
%The subgraphs summarization network with hyper-parameters $(r_1,r_2\ldots,r_t) \in \N^t$ is equivalent (in terms of complexity) to working with tensors of order
%\begin{align}
%k=1+\frac{tr_1\log(\Delta) +\log(r_1+1)}{\log(n)},
%\end{align}
%for graphs with max degree $\Delta$.
%\end{corollary}
\fi

The above results show that when using RNP-GNNs with sparse graphs, we can represent functions of  substructures with $k$ nodes without requiring $k-$order tensors.
LRPs also encode neighborhoods of distance $r_1$ around nodes. In particular, all $c!$ permutations of the nodes in a neighborhood of size $c$ are considered to obtain the representation. As a result, %if $c = \Omega(\log(n))$, then LRPs are more expensive than $O(n^k)$. Moreover,
LRP networks only have polynomial complexity if $c = o(\log(n))$. Thus, RNP-GNNs can provide an exponential improvement in terms of the tolerable size $c$ of neighborhoods with distance $r_1$ in the graph. 

%LRPs are more expensive, since $\log(c! )= \Theta(c\log(c))$ which is much greater than $\log(n^k) = k \log(n)$ if $c = \Omega(\log(n))$. As a result, LRP networks only have polynomial complexity if $c = o(\log(n))$. Thus, RNP-GNNs provide an exponential improvement in terms of the tolerable size of neighborhoods with distance $r_1$ in the graph. 

Moreover, Theorem~\ref{th3} suggests to aim for small $r_1$. The other $r_i$'s may be larger than $r_1$, as shown in Figure~\ref{fig4}, but do not affect the upper bound on the complexity. %In the appendix, we show that it is possible to search for a covering sequence with minimum $r_1$.
%, and hence as long as $r_1$ is small we have low complexity. Note that the covering sequence for graphs may not be increasing. For example, see Figure 4.

%\begin{figure}[t]
%\centering
%\includegraphics[scale = 0.8]{05}
%\caption{For the above graph, $(v_1,v_2,\ldots,v_6)$ is a vertex covering sequence. The  corresponding covering sequence $(1,4,3,2,1)$ is not decreasing. \label{fig4}  }
%\end{figure}

\section{An Information-Theoretic Lower Bound}

In this section, we provide a general information-theoretic lower bound for graph representations that encode a given graph $G$ by first encoding a number of (possibly small) graphs $G_1,G_2,\ldots,G_t$ and then aggregating the resulting representations. The sequence of graphs $G_1,G_2,\ldots,G_t$ may be obtained in an arbitrary way from $G$. For example, in an MPNN, $G_i$ can  be the computation tree (rooted tree) at node $i$. As another example, in LRP, $G_i$ is the local neighborhood around node $i$.

Formally, consider a graph representation $f(.;\theta):\G_n \to \R^d$ as 
\begin{align}
f(G;\theta) = \textsc{Aggregate}( \lb \psi(G_i): i \in [t] \rb),\quad [t] =\{1, \ldots, t\}
\end{align}
for any $G \in \G_n$, where \textsc{Aggregate} is a multi-set function,  $(G_1,G_2,\ldots,G_t) = \Xi (G)$ where 
 $\Xi(.):\G_n \to \big (\bigcup_{m=1}^\infty \G_m \big)^t$ is a function from one graph to $t$ graphs, and $\psi: \bigcup_{m=1}^\infty \G_m \to [s]$ is a function on graphs taking $s$ values. In short, we encode $t$ graphs, and each encoding takes one of $s$ values. We call this graph representation function an $(s,t)$-good graph representation.

 \begin{theorem}\label{theorem4}
 Consider a parametrized class of $(s,t)-$good representations $f(.;\theta):\G_n \to \R^d$ that is able to count any (not necessarily induced\footnote{The theorem also holds for induced subgraphs, with/without node attributes.}) substructure with $k$ vertices. More precisely, for any graph $H$ with $k$ vertices, there exists $f(.;\theta)$ such that if $C(G_1;H)\neq C(G_2;H)$, then  $f(G_1;\theta) \neq f(G_2;\theta)$.
 Then\footnote{$\tilde{\Omega}(m)$ is $\Omega(m)$ up to poly-logarithmic  factors.  }  $t = \tilde{\Omega}(n^{\frac{k}{s-1}})$. 
 \end{theorem}
 
In particular, for any $(s,t)-$good graph  representation with $s=2$, i.e., binary encoding functions, we need $\tilde{\Omega}(n^k)$ encoded graphs. This implies that, for $s=2$, enumerating all subgraphs and deciding for each whether it equals $H$ is near optimal. Moreover, if $s\le k$, then $t=\Theta(n)$ small graphs would not suffice to enable counting. 

More interestingly, if $k,s = O(1)$, then it is impossible to perform the substructure counting task with $t= O(\log(n))$. As a result, in this case, considering $n$ encoded graphs (as is done in GNNs or LRP networks) cannot be exponentially improved. 

The lower bound in this section is   information-theoretic and hence applies to any algorithm. It may be possible to strengthen it by considering computational complexity, too. For binary encodings, i.e., $s=2$, however, we know that the bound cannot be improved since manual counting of subgraphs matches the lower bound. 

%, i.e., by considering computational complexity consideration it may be possible to improve it. For binary encoding (i.e., $s=2$); however, we know that the bound cannot be improved since manual counting of subgraphs matches the lower bound. 

 \section{Time Complexity Lower Bounds for Counting Subgraphs}
 
 In this section, we put our results in the context of known hardness results for subgraph counting.
%
% provide an overview of  existing 
%lower bounds on the time complexity of subgraph counting. After that, we conclude a number of lower bounds for counting subgraphs via GNNs.
%
In general, the subgraph isomorphism problem  is known to be NP-complete. Going further, 
the Exponential Time Hypothesis (ETH) is a conjecture in complexity theory %about the complexity of the 3-SAT problem
\citep{impagliazzo2001complexity}, and states that several NP-complete problems cannot be solved in sub-exponential time.  ETH, as a stronger version of the $P\neq NP$ problem,  is widely believed to hold.  
Assuming that ETH holds, the $k-$clique detection problem requires at least $n^{\Omega(k)}$ time \citep{chen2005tight}. This means that if a graph representation can count \textit{any} subgraph $H$ of size $k$, then computing it requires at least $n^{\Omega(k)}$ time.
%This immediately shows that to count substructures on $k-$vertices, the impractical time complexity $n^{\Omega(k)}$ is unavoidable. 
%
%
%
\begin{corollary}\label{cor:eth}
  Assuming the ETH conjecture holds, any graph representation that can count any substructure $H$ on $k$ vertices with appropriate parametrization %(possibly depending on $H$)
  needs $n^{\Omega(k)}$ time to compute. 
\end{corollary}
The above bound matches the $O(n^k)$ complexity of the higher-order GNNs. Comparing with Theorem~\ref{theorem4} above, Corollary~\ref{cor:eth} is more general, while Theorem~\ref{theorem4} has fewer assumptions and offers a refined result for aggregation-based graph representations.

Given that Corollary~\ref{cor:eth} is a \textit{worst-case} bound, a natural question is whether we can do better for subclasses of graphs. Regarding $H$, even if $H$ is a random Erd\"os-R\'enyi graph, it can only be counted in  $n^{\Omega(k/\log{k})}$ time \citep{dalirrooyfard2019graph}.

%One may ask if we choose a graph on $k-$vertices $H$ at random, can we get alleviate the impractical $n^{\Omega(k)}$ time complexity? Unfortunately, under famous Hadwiger’s conjecture it is recently shown  that  an Erd\H os-R\'enyi random graph  $H$  on $k-$vertices can be counted only in time $n^{\Omega(k/\log{k})}$, which is still impractical for graph representation learning \citep{dalirrooyfard2019graph}.

%As we explained, the worst-case time complexity of the RNP-GNN algorithm can be as large as $\Omega(n^{k})$. However, the provided explanations in this section shows that it is near optimal for the counting task in worst-case regime. For sparse graphs, in contrast, RNP-GNNs   achieve almost linear complexity. This shows that the RNP-GNN algorithm   theoretically \textit{scales} when the input graph is sparse. This advantage is not achieved with the current higher-order representations such as $k-$GNNs and $k-$IGNs where the complexity is does not reduce when the input graph is sparse.

Regarding the input graph in which we count, consider two classes  of sparse graphs: \emph{strongly sparse graphs} have maximum degree  $\Delta  = O(1)$, and \emph{weakly sparse graphs} have average degree $\bar{\Delta} = O(1)$.  We argued in Theorem~\ref{th3} that RNP-GNNs achieve almost \textit{linear} complexity for the class of strongly sparse graphs. For weakly sparse graphs, in contrast, the complexity of RNP-GNNs is generally not linear, but still polynomial, and can be much better  than $O(n^k)$.  One may ask whether it is possible to achieve a learnable graph representation such that its complexity for weakly sparse graphs is still linear.
%of it restricted to the weakly sparse graphs is still linear.
Recent results in complexity theory imply that this is impossible:
\begin{corollary}[\cite{gishboliner2020counting, bera2019linear, bera2020nearlinear}]
There is no graph representation algorithm that runs in linear time on weakly sparse graphs and is able  to count any substructure $H$ on $k$ vertices (with appropriate parametrization). % (possibly depending on $H$). 
\end{corollary}
Hence, RNP-GNNs are close to optimal for several cases of  counting substructures with parametrized learnable functions. 
%This shows another near-optimality result for RNP-GNNs for counting substructures with parametrized learnable functions. 

\section{Experiments}

In this section, we validate our theoretical findings via numerical experiments.
Here, we briefly describe our experimental setup and results --- further experimental details are given in Appendix~\ref{appendix:experiments}.

\begin{table}[ht]
\vspace{-8pt}
\parbox{.55\linewidth}{
    \centering
    \caption{\small Numerical results for counting induced triangles and non-induced $3$-stars, following the setup of \cite{chen2020can}. We report the test MSE divided by variance of the true counts of each substructure (lower is better). The best three models for each task are bolded. 
    }
    \vspace{1mm}
    \label{tab:counting1}
    \scalebox{0.8}{
    \begin{tabular}{lccccc}
    \toprule
         & \multicolumn{2}{c}{Erd\H{o}s-Renyi} &&  \multicolumn{2}{c}{Random Regular} \\
         \cmidrule[0.25pt]{2-3}\cmidrule[0.25pt]{5-6}
         & triangle &  $3$-star && triangle &  $3$-star \\
        \midrule
        GCN  & 6.78E-1  &  4.36E-1  && 1.82  & 2.63  \\
        GIN  & 1.23E-1  &  1.62E-4  && 4.70E-1  & 3.73E-4  \\
       GraphSAGE & 1.31E-1  & \bf 2.40E-10  && 3.62E-1  & \bf 8.70E-8 
       \\
       sGNN  & 9.25E-2  &  2.36E-3  && 3.92E-1  & 2.37E-2   \\
       2-IGN  & 9.83E-2  &  5.40E-4  && 2.62E-1  & 1.19E-2   \\
       PPGN & \bf 5.08E-8  &  4.00E-5  && \bf 1.40E-6  & 8.49E-5  \\
       LRP-1-3 & 1.56E-4  &  2.17E-5  && 2.47E-4  & \bf 1.88E-6  \\
       Deep LRP-1-3 & \bf 2.81E-5  &  \bf 1.12E-5  && \bf 1.30E-6  & \bf 2.07E-6  \\
       RNP-GNN & \bf 1.39E-5 & \bf 1.39E-5 && \bf 2.38E-6 & 1.50E-4 \\
      \bottomrule
    \end{tabular}
    }
    }
    \hfill
    \parbox{.4\textwidth}{
    \centering
    \vspace{-10pt}
    \caption{\small Test accuracy on the EXP dataset with setup as in \cite{abboud2021surprising}. 
    %\sj{%Say what the other models are that are not in Table 1. 
    %Numbers copied from Abboud?}
    }
    \label{tab:exp}
    \scalebox{0.8}{
    \begin{tabular}{ll}
    \toprule
         Model & Accuracy (\%)  \\
         \midrule
         GCN-RNI & 98.0 $\pm$ 1.85  \\
         PPGN & 50.0 \\
         1-2-3-GCN-L & 50.0 \\
         3-GCN & \bf 99.7 $\pm$ 0.004 \\
         RNP-GNN $(r_1=1)$ &  50.0 \\
         RNP-GNN $(r_1=2)$ & \bf 99.8 $\pm$ 0.005 \\
         \bottomrule
    \end{tabular}
    }
    }
\end{table}
%\setlength\intextsep{\oldintextsep}

%\setlength\intextsep{0pt}
%\begin{wraptable}{r}{.5\textwidth}
%\begin{table}[ht]
    % \tiny
%    \centering
%    \caption{\small Numerical results for induced subgraph count of tailed triangles and chordal cycle, with the setup of \cite{chen2020can}. \textcolor{red}{Note: We are going to remove this table.}}
%    \label{tab:counting2}
%    \vspace{5pt}
%    \scalebox{0.82}{
%    \begin{tabular}{@{}lccccccccc@{}}
%    \toprule
%         & \multicolumn{2}{c}{Erd\H{o}s-Renyi} &&  \multicolumn{2}{c}{Random Regular} \\
%         \cmidrule[0.25pt]{2-3}\cmidrule[0.25pt]{5-6}
%         & Tailed Triangle &  Chordal Cycle && Tailed Triangle & Chordal Cycle \\ 
%        \midrule
%       PPGN & 7.11E-3 & 2.14E-2 && 2.29E-3 & 5.90E-4 \\
%       LRP-1-3 & 7.61E-5 & 5.97E-4 && 9.80E-5 & 8.19E-5 \\
%       Deep LRP-1-3 & \bf 3.00E-6 & \bf 8.03E-6 && \bf 1.37E-7 & \bf 7.54E-13 \\
%       RNP-GNN & 3.99E-3 & 8.91E-4 && 9.01E-4 & 6.83E-4 \\
%      \bottomrule
%    \end{tabular}
%    }
%\end{wraptable}
%\end{table}
%\setlength\intextsep{\oldintextsep}

\textbf{Counting substructures.}
First, we follow the experimental setup of \cite{chen2020can} on tasks for counting substructures. In Table~\ref{tab:counting1}, we report results for learning the induced subgraph count of triangles and non-induced subgraph count of $3$-stars. Our RNP-GNN model is consistently within the best performing models for these counting tasks, thus validating our theoretical results. Based on the baseline results taken from \citep{chen2020can}, RNP-GNN tends to widely outperform MPNNs (GCN~\citep{kipf2017semi}, GIN~\citep{xu2018powerful}, GraphSAGE~\citep{hamilton2017inductive}), and other models not tailored for counting: spectral GNN~\citep{chen2018supervised}, and 2-IGN~\citep{maron2018invariant}. Also, RNP-GNN often beats higher-order GNNs: PPGN~\citep{maron2019provably} and LRP-1-3~\citep{chen2020can}. RNP-GNN is mostly comparable to Deep LRP-1-3, though Deep LRP-1-3 outperforms it in a few cases. Recall that Deep LRP-1-3 is a practical version of LRP --- we leave further developments of practical variants of RNP-GNN to future work. 
%\sj{Would be helpful to briefly ``classify'' the comparison models so that the reader knows what they are. E.g., We compare to MPNN (GCN,GIN), higher-order networks (\ldots) etc. If we need space, move the two tables next to each other.}

\textbf{Satisfiability of propositional formulas.} Second, we test the expressiveness of our model in distinguishing non-isomorphic graphs that 1-WL cannot distinguish. The EXP dataset \citep{abboud2021surprising} for classifying whether certain propositional formulas are satisfiable requires higher than 1-WL expressive power to achieve better than random accuracy. As shown in Table~\ref{tab:exp}, while our RNP-GNN with $r_1 = 1$ is unable to achieve better than random accuracy, our RNP-GNN with $r_1 = 2$ achieves near perfect accuracy --- beating all other models based on results taken from~\citep{abboud2021surprising}. These other models include universal models with random node identifiers (GCN-RNI~\citep{abboud2021surprising}), GNNs with 3-WL power (PPGN~\citep{maron2019provably}), and GNNs that imitate some (possibly weaker) version of 3-WL (1-2-3-GCN-L~\citep{morris2019weisfeiler}, 3-GCN~\citep{abboud2021surprising}). Thus, our architecture, which is not developed within common frameworks for achieving $k$-WL expressiveness, is in fact powerful at distinguishing non-isomorphic graphs.

  \section{Conclusion}\label{conclusion}
In this paper, we studied the theoretical possibility of counting substructures (induced subgraphs) by a graph neural network that pools encodings of graphs derived from the input graph. In particular, we show that \emph{recursive} pooling can achieve counting power without additional representation techniques via other architectures. This approach is different from previous works, and the insights for how it captures structural information --- neighborhood intersections and marking --- is potentially of more general interest for designing future architectures.
Our second set of results establishes general lower bounds on the complexity of GNNs that can count substructures of a given size, and puts the positive results in context.

\iffalse
We proposed an architecture, called RNP-GNN, and we proved that for reasonably sparse graphs we can efficiently count substructures.  Characterizing the expressive power of GNNs via the set of functions they can learn on substructures may be useful for developing new architectures. In the end, we proved a general lower bound for any graph representation which counts subgraphs and works by aggregating representations of a collection of graphs derived from the graph. 
\fi

%\section*{Acknowledgements}

%This project was funded by NSF CAREER award 1553284 and an ONR MURI.

%\newpage

%\bibliography{iclr_main}

\begin{thebibliography}{10}

\bibitem{scarselli2008graph}
F.~Scarselli, M.~Gori, A.~C. Tsoi, M.~Hagenbuchner, and G.~Monfardini, ``The
  graph neural network model,'' {\em IEEE Transactions on Neural Networks},
  vol.~20, no.~1, pp.~61--80, 2008.

\bibitem{kipf2017semi}
T.~N. Kipf and M.~Welling, ``Semi-supervised classification with graph
  convolutional networks,'' in {\em International Conference on Learning
  Representations}, 2017.

\bibitem{hamilton2017inductive}
W.~Hamilton, Z.~Ying, and J.~Leskovec, ``Inductive representation learning on
  large graphs,'' in {\em Advances in neural information processing systems},
  pp.~1024--1034, 2017.

\bibitem{duvenaud2015convolutional}
D.~K. Duvenaud, D.~Maclaurin, J.~Iparraguirre, R.~Bombarell, T.~Hirzel,
  A.~Aspuru-Guzik, and R.~P. Adams, ``Convolutional networks on graphs for
  learning molecular fingerprints,'' in {\em Advances in neural information
  processing systems}, pp.~2224--2232, 2015.

\bibitem{defferrard2016convolutional}
M.~Defferrard, X.~Bresson, and P.~Vandergheynst, ``Convolutional neural
  networks on graphs with fast localized spectral filtering,'' in {\em Advances
  in neural information processing systems}, pp.~3844--3852, 2016.

\bibitem{battaglia2016interaction}
P.~Battaglia, R.~Pascanu, M.~Lai, D.~J. Rezende, {\em et~al.}, ``Interaction
  networks for learning about objects, relations and physics,'' in {\em
  Advances in neural information processing systems}, pp.~4502--4510, 2016.

\bibitem{jin2018learning}
W.~Jin, K.~Yang, R.~Barzilay, and T.~Jaakkola, ``Learning multimodal
  graph-to-graph translation for molecule optimization,'' in {\em International
  Conference on Learning Representations}, 2018.

\bibitem{gilmer2017neural}
J.~Gilmer, S.~S. Schoenholz, P.~F. Riley, O.~Vinyals, and G.~E. Dahl, ``Neural
  message passing for quantum chemistry,'' in {\em Proceedings of the 34th
  International Conference on Machine Learning-Volume 70}, pp.~1263--1272,
  2017.

\bibitem{xu2018powerful}
K.~Xu, W.~Hu, J.~Leskovec, and S.~Jegelka, ``How powerful are graph neural
  networks?,'' in {\em International Conference on Learning Representations},
  2019.

\bibitem{morris2019weisfeiler}
C.~Morris, M.~Ritzert, M.~Fey, W.~L. Hamilton, J.~E. Lenssen, G.~Rattan, and
  M.~Grohe, ``Weisfeiler and leman go neural: Higher-order graph neural
  networks.,'' in {\em AAAI}, 2019.

\bibitem{chen2020can}
Z.~Chen, L.~Chen, S.~Villar, and J.~Bruna, ``Can graph neural networks count
  substructures?,'' {\em arXiv preprint arXiv:2002.04025}, 2020.

\bibitem{garg2020generalization}
V.~Garg, S.~Jegelka, and T.~Jaakkola, ``Generalization and representational
  limits of graph neural networks,'' in {\em Int. Conference on Machine
  Learning (ICML)}, pp.~5204--5215, 2020.

\bibitem{elton2019deep}
D.~C. Elton, Z.~Boukouvalas, M.~D. Fuge, and P.~W. Chung, ``Deep learning for
  molecular design a review of the state of the art,'' {\em Molecular Systems
  Design and Engineering}, vol.~4, no.~4, pp.~828--849, 2019.

\bibitem{sun2020graph}
M.~Sun, S.~Zhao, C.~Gilvary, O.~Elemento, J.~Zhou, and F.~Wang, ``Graph
  convolutional networks for computational drug development and discovery,''
  {\em Briefings in bioinformatics}, vol.~21, no.~3, pp.~919--935, 2020.

\bibitem{loukas2019graph}
A.~Loukas, ``What graph neural networks cannot learn: depth vs width,'' in {\em
  International Conference on Learning Representations}, 2019.

\bibitem{sato2019approximation}
R.~Sato, M.~Yamada, and H.~Kashima, ``Approximation ratios of graph neural
  networks for combinatorial problems,'' in {\em Advances in Neural Information
  Processing Systems}, pp.~4081--4090, 2019.

\bibitem{abboud2021surprising}
R.~Abboud, A.~A. Ceylan, M.~Grohe, and T.~Lukasiewicz, ``The surprising power
  of graph neural networks with random node initialization,'' in {\em
  Proceedings of the Thirtieth International Joint Conference on Artificial
  Intelligence, {IJCAI-21}}, pp.~2112--2118, 8 2021.

\bibitem{maron2018invariant}
H.~Maron, H.~Ben-Hamu, N.~Shamir, and Y.~Lipman, ``Invariant and equivariant
  graph networks,'' in {\em International Conference on Learning
  Representations}, 2018.

\bibitem{maron2019provably}
H.~Maron, H.~Ben-Hamu, H.~Serviansky, and Y.~Lipman, ``Provably powerful graph
  networks,'' in {\em Advances in Neural Information Processing Systems
  (NeurIPS)}, pp.~2156--2167, 2019.

\bibitem{azizian2020characterizing}
W.~Azizian and M.~Lelarge, ``Characterizing the expressive power of invariant
  and equivariant graph neural networks,'' {\em arXiv preprint
  arXiv:2006.15646}, 2020.

\bibitem{maron2019universality}
H.~Maron, E.~Fetaya, N.~Segol, and Y.~Lipman, ``On the universality of
  invariant networks,'' in {\em International Conference on Machine Learning},
  pp.~4363--4371, 2019.

\bibitem{keriven2019universal}
N.~Keriven and G.~Peyr{\'e}, ``Universal invariant and equivariant graph neural
  networks,'' in {\em Advances in Neural Information Processing Systems
  (NeurIPS)}, pp.~7092--7101, 2019.

\bibitem{kiefer2020iteration}
S.~Kiefer and B.~D. McKay, ``The iteration number of colour refinement,'' in
  {\em 47th International Colloquium on Automata, Languages, and Programming
  (ICALP 2020)}, Schloss Dagstuhl-Leibniz-Zentrum f{\"u}r Informatik, 2020.

\bibitem{murphy2019relational}
R.~Murphy, B.~Srinivasan, V.~Rao, and B.~Riberio, ``Relational pooling for
  graph representations,'' in {\em International Conference on Machine Learning
  (ICML 2019)}, 2019.

\bibitem{murphy2019janossy}
R.~Murphy, B.~Srinivasan, V.~Rao, and B.~Riberio, ``Janossy pooling: Learning
  deep permutation-invariant functions for variable-size inputs,'' in {\em
  International Conference on Learning Representations}, 2019.

\bibitem{scarselli2009}
F.~Scarselli, M.~Gori, A.~C. Tsoi, M.~Hagenbuchner, and G.~Monfardini,
  ``Computational capabilities of graph neural networks,'' {\em IEEE
  Transactions on Neural Networks}, vol.~20, no.~1, pp.~81--102, 2009.

\bibitem{chen2019equivalence}
Z.~Chen, S.~Villar, L.~Chen, and J.~Bruna, ``On the equivalence between graph
  isomorphism testing and function approximation with gnns,'' in {\em Advances
  in Neural Information Processing Systems}, pp.~15894--15902, 2019.

\bibitem{liu2019n}
S.~Liu, M.~F. Demirel, and Y.~Liang, ``N-gram graph: Simple unsupervised
  representation for graphs, with applications to molecules,'' in {\em Advances
  in Neural Information Processing Systems}, pp.~8466--8478, 2019.

\bibitem{monti2018motifnet}
F.~Monti, K.~Otness, and M.~M. Bronstein, ``Motifnet: a motif-based graph
  convolutional network for directed graphs,'' in {\em 2018 IEEE Data Science
  Workshop (DSW)}, pp.~225--228, IEEE, 2018.

\bibitem{liu2020neural}
X.~Liu, H.~Pan, M.~He, Y.~Song, X.~Jiang, and L.~Shang, ``Neural subgraph
  isomorphism counting,'' in {\em Proceedings of the 26th ACM SIGKDD
  International Conference on Knowledge Discovery \& Data Mining},
  pp.~1959--1969, 2020.

\bibitem{yu2020sumgnn}
Y.~Yu, K.~Huang, C.~Zhang, L.~M. Glass, J.~Sun, and C.~Xiao, ``Sumgnn:
  Multi-typed drug interaction prediction via efficient knowledge graph
  summarization,'' {\em arXiv preprint arXiv:2010.01450}, 2020.

\bibitem{meng2018subgraph}
C.~Meng, S.~C. Mouli, B.~Ribeiro, and J.~Neville, ``Subgraph pattern neural
  networks for high-order graph evolution prediction.,'' in {\em AAAI},
  pp.~3778--3787, 2018.

\bibitem{cotta2020unsupervised}
L.~Cotta, C.~H.~C. Teixeira, A.~Swami, and B.~Ribeiro, ``Unsupervised joint
  $k$-node graph representations with compositional energy-based models,'' {\em
  arXiv preprint arXiv:2010.04259}, 2020.

\bibitem{alsentzer2020subgraph}
E.~Alsentzer, S.~Finlayson, M.~Li, and M.~Zitnik, ``Subgraph neural networks,''
  {\em Advances in Neural Information Processing Systems}, vol.~33, 2020.

\bibitem{huang2020graph}
K.~Huang and M.~Zitnik, ``Graph meta learning via local subgraphs,'' {\em arXiv
  preprint arXiv:2006.07889}, 2020.

\bibitem{rex2020neural}
R.~Ying, Z.~Lou, J.~You, C.~Wen, A.~Canedo, and J.~Leskovec, ``Neural subgraph
  matching,'' {\em arXiv preprint arXiv:2007.03092}, 2020.

\bibitem{bouritsas2020improving}
G.~Bouritsas, F.~Frasca, S.~Zafeiriou, and M.~M. Bronstein, ``Improving graph
  neural network expressivity via subgraph isomorphism counting,'' {\em arXiv
  preprint arXiv:2006.09252}, 2020.

\bibitem{vignac2020building}
C.~Vignac, A.~Loukas, and P.~Frossard, ``Building powerful and equivariant
  graph neural networks with message-passing,'' {\em arXiv preprint
  arXiv:2006.15107}, 2020.

\bibitem{you2021identity}
J.~You, J.~M. Gomes-Selman, R.~Ying, and J.~Leskovec, ``Identity-aware graph
  neural networks,'' in {\em Proceedings of the AAAI Conference on Artificial
  Intelligence}, vol.~35, pp.~10737--10745, 2021.

\bibitem{sandfelder2021ego}
D.~Sandfelder, P.~Vijayan, and W.~L. Hamilton, ``Ego-gnns: Exploiting ego
  structures in graph neural networks,'' in {\em ICASSP 2021-2021 IEEE
  International Conference on Acoustics, Speech and Signal Processing
  (ICASSP)}, pp.~8523--8527, IEEE, 2021.

\bibitem{cotta2021reconstruction}
L.~Cotta, C.~Morris, and B.~Ribeiro, ``Reconstruction for powerful graph
  representations,'' 2021.

\bibitem{kelly1957congruence}
P.~Kelly {\em et~al.}, ``A congruence theorem for trees.,'' {\em Pacific
  Journal of Mathematics}, vol.~7, no.~1, pp.~961--968, 1957.

\bibitem{impagliazzo2001complexity}
R.~Impagliazzo and R.~Paturi, ``On the complexity of k-sat,'' {\em Journal of
  Computer and System Sciences}, vol.~62, no.~2, pp.~367--375, 2001.

\bibitem{chen2005tight}
J.~Chen, B.~Chor, M.~Fellows, X.~Huang, D.~Juedes, I.~A. Kanj, and G.~Xia,
  ``Tight lower bounds for certain parameterized np-hard problems,'' {\em
  Information and Computation}, vol.~201, no.~2, pp.~216--231, 2005.

\bibitem{dalirrooyfard2019graph}
M.~Dalirrooyfard, T.~D. Vuong, and V.~V. Williams, ``Graph pattern detection:
  Hardness for all induced patterns and faster non-induced cycles,'' in {\em
  Proceedings of the 51st Annual ACM SIGACT Symposium on Theory of Computing},
  pp.~1167--1178, 2019.

\bibitem{gishboliner2020counting}
L.~Gishboliner, Y.~Levanzov, and A.~Shapira, ``Counting subgraphs in degenerate
  graphs,'' {\em arXiv preprint arXiv:2010.05998}, 2020.

\bibitem{bera2019linear}
S.~K. Bera, N.~Pashanasangi, and C.~Seshadhri, ``Linear time subgraph counting,
  graph degeneracy, and the chasm at size six,'' {\em arXiv preprint
  arXiv:1911.05896}, 2019.

\bibitem{bera2020nearlinear}
S.~K. Bera, N.~Pashanasangi, and C.~Seshadhri, ``Near-linear time homomorphism
  counting in bounded degeneracy graphs: The barrier of long induced cycles,''
  2020.

\bibitem{chen2018supervised}
Z.~Chen, L.~Li, and J.~Bruna, ``Supervised community detection with line graph
  neural networks,'' in {\em International Conference on Learning
  Representations}, 2018.

\bibitem{mckay1997small}
B.~D. McKay, ``Small graphs are reconstructible,'' {\em Australasian Journal of
  Combinatorics}, vol.~15, pp.~123--126, 1997.

\bibitem{hornik1989multilayer}
K.~Hornik, M.~Stinchcombe, H.~White, {\em et~al.}, ``Multilayer feedforward
  networks are universal approximators.,'' {\em Neural networks}, vol.~2,
  no.~5, pp.~359--366, 1989.

\bibitem{hornik1991approximation}
K.~Hornik, ``Approximation capabilities of multilayer feedforward networks,''
  {\em Neural networks}, vol.~4, no.~2, pp.~251--257, 1991.

\bibitem{kleinberg2006algorithm}
J.~Kleinberg and E.~Tardos, {\em Algorithm design}.
\newblock Pearson Education India, 2006.

\bibitem{erdos1960evolution}
P.~Erdos, A.~R{\'e}nyi, {\em et~al.}, ``On the evolution of random graphs,''
  {\em Publ. Math. Inst. Hung. Acad. Sci}, vol.~5, no.~1, pp.~17--60, 1960.

\bibitem{steger1999generating}
A.~Steger and N.~C. Wormald, ``Generating random regular graphs quickly,'' {\em
  Combinatorics, Probability and Computing}, vol.~8, no.~4, pp.~377--396, 1999.

\end{thebibliography}
%\bibliographystyle{iclr2022_conference}
%\bibliographystyle{ieeetr}

%\newpage

\appendix

%\input{more_related_works}

 %\section{Proofs}
\section{Proof of Theorem \ref{th1}}

\subsection{Preliminaries}\label{sec_pre}

Let us first state a few definitions about the graph functions. Note that for any graph function  $f:\G_n \to \R^d$, we  have  $f(G)= f(H)$ for any  $G \cong H$. 

\begin{definition}
Given two graph functions $f,g:\G_n\to \R^d$, we write $f \sqsupseteq g   $, if and only if for any $ G_1,G_2 \in \G_n$,
\begin{align}
\forall G_1,G_2\in G_n : 
g(G_1) \neq g(G_2) \implies  f(G_1) \neq f(G_2),  
\end{align}
or, equivalently, 
\begin{align}
\forall G_1,G_2\in G_n : 
f(G_1) = f(G_2) \implies  g(G_1) = g(G_2).  
\end{align}
\end{definition}

\begin{proposition} 
Consider graph functions $f,g,h:\G_n \to \R^d$ such that $ f \sqsupseteq g$ and $ g \sqsupseteq h$. Then, $ f \sqsupseteq h$. In other words, $\sqsupseteq$ is transitive. 
\end{proposition}

\begin{proof}
The proposition holds by definition. 
\end{proof}

\begin{proposition} \label{prop_combine}
Consider graph functions $f,g:\G_n \to \R^d$ such that $ f \sqsupseteq g$. Then, there is a function $\xi: \R^d \to \R^d$ such that $\xi \circ f = g$.
\end{proposition}

\begin{proof}
Let $\G_n = \sqcup_{i \in \N} \calF_i$ be the partitioning induced by the equality relation with respect to the function $f$ on $\G_n$. Similarly define $\calG_i$, $i \in \N$ for $g$. Note that due to the definition, $\{ \calF_i : i \in \N \}$ is a refinement for $\{\calG_i:  i \in \N \}$. Define $\xi$ to be the unique mapping from $\{ \calF_i : i \in \N \}$ to $\{\calG_i:  i \in \N \}$ which respects the equality  relation. One can observe that such $\xi$ satisfies the requirement in the proposition. 
\end{proof}

\begin{definition}
An RNP-GNN is called maximally expressive, if and only if 
\begin{itemize}
\item all the aggregate functions are injective as  mappings from a multi-set on a countable ground set to their codomain. 
\item all the combine functions are injective mappings. 
\end{itemize}
\end{definition}

\begin{proposition} 
 Consider two RNP-GNNs $f,g$ with the same recursion parameters $\bold{r} = (r_1,r_2,\ldots,r_\tau)$ where $f$ is maximally expressive. Then, $f \sqsupseteq g$.
\end{proposition}

\begin{proof}
The proposition holds by definition.
\end{proof}

\begin{proposition}\label{prop1}
Consider a sequence of graph functions $f,g_1,\ldots,g_k$. If $f \sqsupseteq g_{i}$ for all $i \in [k]$, then 
\begin{align}
f \sqsupseteq \sum_{i=1}^k c_i g_{i},
\end{align}
for any $c_i \in \R$, $i \in \N$.
\end{proposition}

\begin{proof}
Since $f \sqsupseteq g_{i}$, we have
\begin{align}
\forall G_1,G_2\in G_n : 
f(G_1) = f(G_2) \implies  g_i(G_1) = g_i(G_2), 
\end{align}
for all $i \in [k]$. This means that for any $G_1,G_2\in \G_n$ if $f(G_1) = f(G_2)$ then $ g_i(G_1) = g_i(G_2)$,  $i\in [k]$,  and consequently  $\sum_{i=1}^k c_i g_{i}(G_1) = \sum_{i=1}^k c_ig_{i} (G_2)$. Therefore, from the definition we conclude $f \sqsupseteq \sum_{i=1}^k c_ig_{i}$. Note that the same proof also holds in the case of countable summations as long as the summation is  bounded. 
\end{proof}

\begin{definition}
Let $H=(\calV_H,\calE_H,X^H)$ be a attributed connected simple graph with $k$ nodes. For any attributed graph $G =(\calV_G,\calE_G,X^G)\in \G_n$, the induced subgraph count function $C(G;H)$  is defined as 
\begin{align}
C(G;H):=\sum_{\calS \subseteq [n]} \ind\{G(S) \cong H\}.
\end{align}
Also, let $\bar{C}(G;H)$ denote the number of  non-induced subgraphs of $G$ which are isomorphic to $H$. It can be defined with the homomorphisms from $H$ to $G$. Formally, if $n>k$  define
\begin{align}
\bar{C}(G;H):=\sum_{\calS \subseteq [n] \atop |\calS| = k} \bar{C}(G(\calS);H). 
\end{align}
Otherwise, $n=k$, and we define
\begin{align}
\bar{C}(G;H):=\sum_{\tilde{H} \in \tilde{\calH}(H) } c_{\tilde{H},H} \times  \ind \{ G \cong \tilde{H}\},
%\# \big \{& \varphi : [m] \to [m] :  \forall u,v \in [m] :  (u,v) \in \calE_H \implies (\varphi(u), \varphi(v)) \in \calE_G, \\ & \forall v \in [m]: X^G_{\varphi(v)} = X^H_v \big \}, 
\end{align}
%where $\aut(H)$ denotes the group of automorphisms of $H$ as a labeled graph.
where
\begin{align}
\tilde{\calH}(H):=\{ \tilde{H} \in \G_k : \tilde{H} \Supset H \big \},
\end{align}
is defined with respect to the graph isomorphism, and $c_{\tilde{H},H} \in \N$ denotes the number of subgraphs in $H$ identical to $\tilde{H}$. Note that $\tilde{\calH}(H)$ is a finite set and $\Supset$ denotes  being a (not necessarily induced) subgraph. 
\end{definition}

\begin{proposition}
Let $\calH$ be a family of graphs. If for any $H \in \calH$, there is an  RNP-GNN $f_H(.;\theta)$ with recursion parameters $(r_1,r_2,\ldots,r_\tau)$ such that $f_H \sqsupseteq C(G;H)$, then there exists an RNP-GNN $f(.;\theta)$ with recursion parameters $(r_1,r_2,\ldots,r_\tau)$ such that $f \sqsupseteq \sum_{H \in \calH}C(G;H)$.  
\end{proposition}

\begin{proof}
Let $f(.;\theta)$ be a maximally expressive RNP-GNN. Note that by the definition $f \sqsupseteq f_H$ for any $H \in \calH$. Since $\sqsupseteq$ is transitive,   $f \sqsupseteq C(G;H)$ for all $H \in \calH$, and using Proposition \ref{prop1}, we conclude that $f \sqsupseteq \sum_{H \in \calH}C(G;H)$. 
\end{proof}

The following proposition shows that there is no difference between counting induced attributed graphs and counting induced unattributed graphs in RNP-GNNs. %Also, note that the networks does not depend on the input graph $G$; they only depend on $H$ which is encoded in the function representing the subgraphs summarization network. 

\begin{proposition}
Let $H_0$ be an unattributed connected graph. Assume that for any attributed graph $H$, which is constructed by adding  arbitrary attributes to $H_0$, there exists an  RNP-GNN $f_H(.;\theta_H)$ such that $f_H \sqsupseteq C(G;H)$, then for its unattributed counterpart $H_0$, there exists an RNP-GNN $f(.;\theta)$ with the same recursion parameters as $f_H(.;\theta_H)$ such that $f \sqsupseteq C(G;H_0)$.
\end{proposition}

\begin{proof}
If there exists an RNP-GNN $f_H(.;\theta_H)$ such that $f_H \sqsupseteq C(G;H)$, then for a  maximally expressive  RNP-GNN  $f(.;\theta)$ with the same recursion parameters as $f_H$   we also  have $f \sqsupseteq C(G;H)$. Let $\calH$ be the set of all attributed graphs $H=(\calV,\calE,X) \in \G_k$ up to graph isomorphism, where $X \in \calX^{k}$ for a countable set $\calX$. Note that $\calH = \{H_1,H_2,\ldots\}$ is a  countable set. 
Now we write
\begin{align}
C(G;H_0)&=\sum_{\calS \subseteq [n] \atop |\calS| = k} \ind\{G(S) \cong H_0\}\\
& = \sum_{\calS \subseteq [n] \atop |\calS| = k} \sum_{i \in \N} \ind\{G(S) \cong H_i\}\\
& =  \sum_{i \in \N} \sum_{\calS \subseteq [n] \atop |\calS| = k} \ind\{G(S) \cong H_i\}\\
& =  \sum_{i \in \N} C(G;H_i).\\
\end{align}
%where $H_i$, $i \in \N$ are different possible  labelings for $H_0$, where the labellings are different if their resulting graphs are not isomorphic.   
Now using Proposition \ref{prop1} we conclude that $f \sqsupseteq C(G;H_0)$ since $C(G;H_0)$ is always finite.  
\end{proof}

\begin{definition}
Let   $H$ be a  (possibly attributed) simple connected graph. For any $\calS \subseteq \calV_H$ and $v \in \calV_H$,  define 
\begin{align}
\bar{d}_H(v;\calS) := \max_{u \in \calS} d(u,v).
\end{align}
\end{definition}

\begin{definition}
Let $H$ be a (possibly attributed) connected simple graph with $k=\tau+1$ vertices. A permutation of vertices, such as $(v_1,v_2,\ldots,v_{\tau+1})$,  is called a vertex covering  sequence, with respect to a  sequence $\mathbf{r}  = (r_1,r_2,\ldots,r_{\tau}) \in \N^{\tau}$, called a covering sequence,  if and only if 
\begin{align}
\bar{d}_{H'_i}(v_i;\calS_i) \le r_i,
\end{align}
for $i \in [\tau+1]$, where $H'_i = H(\calS_{i})$ and $\calS_i = \{v_{i},v_{i+1},\ldots,v_{\tau+1}\}$. Let  $\calC_H({\bf r})$  denote the set of all vertex covering  sequences with respect to the covering sequence $\mathbf{r}$ for $H$.
\end{definition}

%\begin{proposition}
%Let   $\bold{r} = (r_1,\bold{r}')$ then
%\begin{align}
%\calC_{H}({\bf r}) =  \calC_{H'_2}({\bf r'}),
%\end{align}
%for any sequence $\bold{r}$. 
%\end{proposition}

\begin{proposition}\label{prop_suffic}
For any $G,H \in \G_k$, if $G \Supset H$ (non-induced subgraph),   then 
\begin{align}
\calC_{H}({\bf r}) \subseteq  \calC_{G}({\bf r}),
\end{align}
for any sequence $\bold{r}$. 
\end{proposition}

\begin{proof}
The proposition follows from the fact that the function $\bar{d}$ is decreasing with introducing new edges. 
%Since we only need to add edges to  $H_2$ in order to achieve $H_1$, all the distances will decrease, and this means that the vertex covering sequences will hold (NEED TO BE CLARIFIED). 
\end{proof}

\begin{proposition}\label{prop_induced}
Assume that Theorem \ref{th1} holds for induced-subgraph count functions. Then, it also holds for the non-induced subgraph count functions. 
\end{proposition}

\begin{proof}
Assume that for a connected (attributed or unattributed) graph $H$, there exists an RNP-GNN with appropriate recursion parameters $f_H(.;\theta_H)$ such that $f_H \sqsupseteq C(G;H)$, then we prove there  exists an RNP-GNN $f(.;\theta)$ with the same recursion parameters as $f_H$ such that  $f \sqsupseteq \bar{C}(G;H)$. 

If there exists  an RNP-GNN $f_H(.;\theta_H)$ such that $f_H \sqsupseteq C(G;H)$, then for a  maximally expressive RNP-GNN  $f(.;\theta)$ with the same recursion parameters as $f_H$   we also have $f \sqsupseteq C(G;H)$. Note that 
\begin{align}
\bar{C}(G,H) &=\sum_{\calS \subseteq [n] \atop |\calS| = k} \bar{C}(G(\calS);H)\\
& =  \sum_{\calS \subseteq [n] \atop |\calS| = k}  \sum_{\tilde{H} \in \tilde{\calH}(H) } c_{\tilde{H},H} \times  \ind \{ G(S) \cong \tilde{H}\} \\
& =  \sum_{\tilde{H} \in \tilde{\calH}(H) } c_{\tilde{H},H}  \sum_{\calS \subseteq [n] \atop |\calS| = k}    \ind \{ G(S) \cong \tilde{H}\} \\
& =  \sum_{i \in \N} c_{H_i,H}  \times C(G,H_i),
\end{align}
where $ \tilde{\calH}(H)=\{ H_1,H_2,\ldots\}$. 
  \begin{claim} \label{claim1}
  $f \sqsupseteq C(G,H_i)$ for any $i$.
  \end{claim}
 Using Proposition \ref{prop1} and  Claim \ref{claim1} we conclude that $f \sqsupseteq \bar{C}(G;H)$ since $\bar{C}(G;H)$ is  finite and $f \sqsupseteq C(G,H_i)$ for any $i$, and the proof is  complete. 
  The missing part which we must show here is that  for any $H_i$ the sequence $(r_1,r_2,\ldots,r_t)$ which covers $H$ also covers $H_i$.  This   follows from Proposition \ref{prop_suffic}. We are done. 
\end{proof}

At the end of this part, let us introduce an important  notation. For any  attributed connected simple graph on $k$ vertices $G=(\calV,\calE,X)$, let $G^*_v$ be the resulting induced  graph obtained after removing $v\in \calV$ from $G$ with the new attributes defined as
  \begin{align}
 X^*_u := (X_u,\ind\{(u,v) \in \calE\}),
\end{align}
for each $u \in \calV\setminus \{v\}$. We may also use $X^{*v}_u$ for more clarification.

\subsection{Proof of Theorem \ref{th1}}\label{app:proof-counting}
We utilize an inductive proof on $\tau$,  which is the length of the covering sequence of $H$. Equivalently, due to the definition, $\tau=k-1$, where $k$ is the number of vertices in $H$.  First, we note that  due to Proposition \ref{prop_induced}, without loss of generality,  we can assume that $H$ is a simple connected attributed graph and the goal is to achieve the   induced-subgraph count function via an RNP-GNN with appropriate recursion parameters. We also consider only maximally expressive networks here to prove the desired result. 

\textbf{Induction base.}  For the induction base,   i.e.,  $\tau=1$,  $H$ is a two-node graph. This means that we only need to count the number of a specific (attributed) edge in the given graph $G$.  Note that in this case we apply an RNP-GNN with recursion parameter $r_1 \ge 1$. Denote the two attributes of the vertices in $H$  by $X^H_1,X^H_2 \in \calX$. The output of an RNP-GNN $f(.;\theta)$ is 
\begin{align}
f(G;\theta) = \phi (\lb    \psi(X^G_v, \varphi (\lb  X^{*v}_u : u \in \calN_{r_1}(v)\rb )) : v \in [n]\rb ), 
\end{align}
where  we assume that $f(.;\theta)$ is  maximally expressive. The goal is to show that $f \sqsupseteq C(G;H)$.  Using the transitivity of $\sqsupseteq$, we only need to choose appropriate $\phi,\psi,\varphi$ to achieve $\hat{f} = C(G;H)$ as the final representation. Let
\begin{align}
\phi( \lb  z_v  : v \in [n]\rb )& :=  \frac{1}{2+ 2 \times \ind\{X_1^H = X_2^H\}} \sum_{i=1}^n z_i\\
\psi(X,(z,z')) & :=z \times  \ind\{ X = X_1^H\} +z' \times  \ind\{ X = X_2^H\} \\
\varphi(\lb  z_u : u \in [n'] \rb ) &:=  \Big( \sum_{i=1}^{n'} \ind \{ z_u = (X^H_2,1)\},  \sum_{i=1}^{n'} \ind \{ z_u = (X^H_1,1)\Big). 
\end{align}
Then, a simple computation shows that  
\begin{align}
\hat{f}(G;\theta) & =  \phi (\lb    \psi(X^G_v, \varphi (\lb  X^{*v}_u : u \in \calN_{r_1}(v)\rb )) : v \in [n]\rb ), \\
&= C(G;H). 
\end{align}
Since $\hat{f}(.;\theta)$ is an RNP-GNN with recursion parameter $r_1$ and for any maximally expressive RNP-GNN $f(.;\theta)$ with the same recursion parameter as $\hat{f}$ we have $f \sqsupseteq \hat{f}$ and $\hat{f} \sqsupseteq C(G;H)$, we conclude that $f \sqsupseteq C(G;H)$ and this completes the proof.

\textbf{Induction step.} Assume that the desired result holds for $\tau-1$ ($\tau \ge2$). We show that it also holds for $\tau$. Let us first define 
\begin{align}
\calH^* &:= \{ H^*_{v_1} :  \exists v_2,\ldots,v_\tau \in [k] :(v_1,v_2,\ldots,v_\tau) \in \calC_{H}(\bold{r}) \} \\
c^* (H^0) &:= \ind\{ H^0 \in \calH^*\}  \times    \# \{ v \in [k] : H^*_v \cong H^0 \},
\end{align}
where $H^*_v$ means the induced subgraph after removing a node, with new attributes (see \ref{sec_pre}).
Note that $\calH^* \neq  \emptyset$ by the assumption.  Let 
\begin{align}
\|\calH^*\|:= \sum _{H^0 \in \calH^*}c^* (H^0) .
\end{align}

For all $H^0 \in \calH^*$, using the induction hypothesis,  there is a (universal) RNP-GNN $\hat{f}(.;\hat{\theta})$ with recursion parameters $(r_2,r_3,\ldots,r_\tau)$ such that $\hat{f} \sqsupseteq C(G;H^0)$. Using Proposition 
\ref{prop1} we conclude 
\begin{align}
\hat{f} \sqsupseteq \sum_{u \in [k]: H^*_u \in \calH^*}   C(G;H^*_u). 
\end{align}
 Define a maximally expressive RNP-GNN with the recursion parameters $(r_1,r_2,\ldots,r_\tau)$ as follows:
\begin{align}
f(G;\theta) = \phi (\lb    \psi(X^G_v, \hat{f}(G^*(\calN_{r_1}(v));\hat{\theta})) : v \in [n]\rb ). 
\end{align}
 Similar to the proof for $\tau=1$, here we only need to propose a (not necessarily maximally expressive) RNP-GNN which achieves the function $C(G;H)$.
 
 Let us define 
 \begin{align}
f_{H^*_u}(G;\theta) := \phi (\lb    \psi_{H^*_u}(X^G_v, \xi \circ \hat{f}(G^*(\calN_{r_1}(v));\hat{\theta})) : v \in [n]\rb ), 
\end{align}
where
 \begin{align}
\phi( \lb  z_v  : v \in [n]\rb )& :=  \frac{1}{\|\calH^*\|} \sum_{i=1}^n z_i\\
\psi_{H^*_u}(X,z) & := z \times \ind\{X = X^H_u\}, \\
\end{align}
 and $\xi \circ \hat{f}  = C(G;H^*_u)$. Note that the existence of such function $\xi$ is guaranteed due to Proposition \ref{prop_combine}. 
Now we write
\begin{align}
\|\calH^*\| \times C(G;H) &= \|\calH^*\| \sum_{\calS \subseteq [n]} \ind\{ G(S) \cong H\} \\
&=\sum_{\calS \subseteq [n]} \sum_{v \in \calS } \ind\{  \exists u \in [k] : (G(S\setminus \{v\}))^*_v  \cong H^*_u\in \calH^*  \land  X^G_v = X^H_u \} \\
&= \sum_{v \in [n] } \sum_{v \in \calS \subseteq [n]}  \ind\{  \exists u \in [k] : (G(S\setminus \{v\}))^*_v  \cong H^*_u\in \calH^*  \land  X^G_v = X^H_u \} \\
&= \sum_{v \in [n] } \sum_{v \in \calS \subseteq  \calN_{r_1}(v)}  \ind\{  \exists u \in [k] :  (G(S\setminus \{v\}))^*_v  \cong H^*_u\in \calH^*  \land  X^G_v = X^H_u \} \\
&= \sum_{v \in [n] } \sum_{v \in \calS \subseteq  \calN_{r_1}(v)} \sum_{u \in [k]: H^*_u \in \calH^*}  \ind\{ (G(S\setminus \{v\}))^*_v  \cong H^*_u\} \ind\{  X^G_v = X^H_u \} \\
& = \sum_{v \in [n]} \sum_{u \in [k]: H^*_u \in \calH^*} C(G^*(\calN_{r_1}(v));H^*_u) \times  \ind\{X^G_v = X^H_u\},
\end{align}
which means that 
\begin{align}
 \sum_{u \in [k]: H^*_u \in \calH^*} f_{H^*_u}(G;\theta) \sqsupseteq C(G;H). 
\end{align}
However, for a maximally expressive RNP-GNN $f(.;\theta)$ we know that $f \sqsupseteq f_{H^*_u}$ for all $H^*_u \in \mathcal{H}$ and this means that $f \sqsupseteq C(G;H)$. The proof is thus complete.

\section{Proof of Theorem \ref{th2}}\label{app:proof-universal}

For any attributed graph $H$ on $r$ nodes (not necessarily connected) we claim that RNP-GNNs can count them.

\begin{claim}
Let $f(.;\theta) :\G_n \to \R^d$ be a maximally expressive RNP-GNN with recursion parameters $(r-1,r-2,\ldots,1)$. Then, $f \sqsupseteq C(G;H)$.  
\end{claim}

Now consider the function 
\begin{align}
\ell(G) = \phi (\lb  \psi(G(S)) : \calS \subseteq \calV, |\calS| \le r \rb  ).
\end{align}
We claim that $f \sqsupseteq \ell$ ($f$ is defined in the previous claim) and this completes the proof according to Proposition \ref{prop_combine}.

To prove the claim, assume that $f(G_1) = f(G_2)$. Then, we conclude that $C(G_1;H) = C(G_2;H)$ for any attributed $H$ (not necessarily connected) with $r$ vertices. Now, we have 
\begin{align}
\ell(G) &= \phi (\lb  \psi(G(S)) : \calS \subseteq \calV, |\calS| \le r \rb  ) \\
& = \phi (\lb  \psi(H) : H \in \G_r, ~\text{the multiplicity of  } ~H~ \text{is}~ C(G;H) \rb  ),
\end{align}
which shows that $\ell(G_1) = \ell(G_2)$.

\textit{Proof of Claim 2.} To prove the claim, we use an induction on the number of connected components  $c_H$ of graph  $H$. If $H$ is connected, i.e.,  $c_H=1$, then according to Theorem \ref{th1}, we know that $f \sqsupseteq C(G;H)$. 

Now assume that the claim holds for $c_H=c-1\ge 1$. We show that it also holds for $c_H=c$.
Let $H_1,H_2,\ldots,H_c$ denote the connected components of $H$. Also assume that $H_i \not \cong H_j$ for all $i\neq j$. We will  relax this assumption  later. 
Let us define
\begin{align}
\calA_G := \{ (\calS_1,\calS_2,\ldots,\calS_c) : \forall i \in [c]: \calS_i \subseteq [n]; G(\calS_i) \cong H_i \}.
\end{align}
Note that we can write
\begin{align}
|\calA_G|&= \prod_{i=1}^c C(G;H_i)\\
&= C(G;H) + \sum_{j=1}^\infty c'_j C(G;H'_j),
\end{align}
where $H'_1,H'_2,\ldots$ are all non-isomorphic graphs obtained by adding edges (at least one edge) between $c$ graphs $H_1,H_2,\ldots,H_c$, or contracting a number of vertices of them. The constants $c'_j$ are just used to remove the effect of multiple counting due to the symmetry. Now, since for any $H_i$, $H'_j$ the number of connected components is strictly less that $c$, using the induction,  we have  $f \sqsupseteq C(G;H_i)$ and  $f \sqsupseteq C(G;H'_j)$ for all $j$ and all $i \in [c]$. According to Proposition \ref{prop1}, we conclude that $f \sqsupseteq C(G;H)$ and this completes the proof. Also, if $H_i$, $i\in [c]$, are not pairwise non-isomorphic, then we can use $\alpha C(G;H)$ in above equation instead of $C(G;H)$, where $\alpha>0$ removes the effect of multiple counting by symmetry. The proof  is thus complete.

%To prove the theorem, first we show that only we need to consider the $r-$local vertex functions. 
%
%\begin{lemma}
%Let $p(G)$ be an $r-$local graph function. Then, there is an $(r-1)-$local vertex function $p'(v)$ such that 
%\begin{align}
%p(G) = \phi'(\lb  p'(v) : v \in \calV\rb ),
%\end{align}
%for an appropriate multi-set function $\phi'$. 
%\end{lemma}
%
% \begin{proof}
% Based on the definition, there exist $\phi,\psi$ such that 
% \begin{align}
%p(G) = \phi (\lb  \psi(G(S)) : \calS \subseteq \calV, |\calS| \le r \rb  ).
%\end{align}
%Let $H \in \G_{\le r}$. If $H$ is not connected, then (by definition) $\psi(H) = 0$. Clearly, 
%\begin{align}
%p(G) &= \phi (\lb  \psi(G(S)) \calS \subseteq \calV, |\calS| \le r \rb  ) \\
% & \le^*
%\{ \{  \psi(G(S)) : \calS \subseteq \calV, |\calS| \le r  \} \} \\
%& \le^* \{ \{  G(S) : \calS \subseteq \calV, |\calS| \le r  \} \},
%\end{align}
%where $\calA \le^* \calB$ denotes the existence of a subjective function from multi-set $\calB$ to  multi-set $\calA$.  Also,
 %\end{proof}

\section{Proof of Theorem \ref{th3}}

To prove Theorem \ref{th3}, we need to bound the number of node updates    required for an RNP-GNN with recursion parameters $(r_1,r_2,\ldots,r_t)$. First of all, we have  $n$   variables used for the final representations of vertices. For each vertex $v_1 \in \calV$, we explore the local neighborhood $\calN_{r_1}(v_1)$ and apply a new RNP-GNN network to that neighborhood. In other words, for the second step we need to update $|\calN_{r_1}(v_1)|$ nodes. Similarly, for the $i$th step of the  algorithm we  have  as most
\begin{align}
\lambda _i :=  \max_{v_1 \in [n]} \max_{  v_{j+1} \in \calN_{r_{j}}(v_{j})  \atop \forall j \in [i-1] } |\calN_{r_1}(v_1) \cap \calN_{r_2}(v_2) \cap \calN_{r_3}(v_3) \ldots \cap \calN_{r_i}(v_{i})|,
\end{align}
updates.  Therefore, we can bound the number of node updates as  
\begin{align}
n \times \prod_{i=1}^\tau \lambda_i.
\end{align}
Since $\lambda_i$ is decreasing in $i$, we simply conclude the desired result.

\section{Proof of Theorem \ref{theorem4}}
 
 Let $K_k$ denote the complete graph on $k$ vertices. 
 
 \begin{claim}\label{claim_countring} For any $k,n \in \N$, such that $n$ is sufficiently large, 
 \begin{align}
 \Big |\{C(G;K_k) : G \in \G_n \} \Big | \ge \frac{(cn/(k\log(n/k))-k)^k}{k!} = \tilde{\Omega}(n^k),
 \end{align}
 where $c$ is a constant which does not depend on $k,n$. 
 \end{claim}

In particular, we claim that the number of different values that $C(G;K_k)$ can take is  $n^k$, up to poly-logarithmic factors.

To prove the theorem, we use the above claim. Consider a class of $(s,t)-$good graph representations $f(.;\theta)$ which can count any substructure on $k$ vertices. As a result, $f \sqsupseteq C(G;K_k)$ for an appropriate parametrization $\theta$.
By the definition, $f(.)$ must take at least $\Big |\{C(G;K_k) : G \in \G_n \} \Big |$ different values, i.e., 
\begin{align}
\Big | \{ f(G;\theta) : G \in \G_n \} \Big | \ge \Big | \{C(G;K_k) : G \in \G_n \} \Big |.
\end{align}
Also,
\begin{align}
\Big | \{ f(G;\theta) : G \in \G_n \} \Big |   \le \Big |    \big \{  \lb  \psi(G_i): i \in [t] \rb : G \in \G_n\big \} \Big |,
\end{align}
where  $(G_1,G_2,\ldots,G_t) = \Xi (G)$. But, $\psi$ can take only $s$ values. Therefore, we have 
\begin{align}
 \Big | \{C(G;K_k) : G \in \G_n \} \Big | &\le
\Big | \{ f(G;\theta) : G \in \G_n \} \Big | \\ & \le \Big |    \big \{  \lb  \psi(G_i): i \in [t] \rb : G \in  \G_n\big \} \Big | \\
&\le  \Big |    \big \{  \lb  \alpha_i: i \in [t] \rb : \forall i \in [t]: \alpha_i \in [s] \} \Big |\\
& \le (t+1)^{s-1}.
\end{align}
As a result, $(t+1)^{s-1} = \tilde{\Omega}(n^k)$ or   $t = \tilde{\Omega}(n^{\frac{k}{s-1}})$. To complete the proof, we only need to prove the claim.

\textit{Proof of Claim \ref{claim_countring}.} 
Let $p_1,p_2,\ldots,p_{m}$ be    distinct  prime numbers less than $n/k$. Using the prime number theorem, we know that $\lim_{n \to \infty} \frac{m}{n/(k\log(n/k))} = 1$. In particular, we can choose $n$ large enough to ensure $cn/(k \log(n/k))<m$ for any constant $c<1$.

For any  ${\calB}=  \{ b_1, b_2,\ldots, b_k\} \subseteq [m] $, define $G_{\calB}$ as a graph on $n$ vertices such that $\calV_{G_{\calB}} = V_0 \sqcup (\sqcup_{i \in [k]} \calV_i)$, and $|\calV_i| = p_{b_i}$. Also,
\begin{align}
e=(u,v) \in G_{\calB} \iff \exists ~{i,j \in [m], i\neq j} : u \in \calV_i ~ \& ~v \in \calV_j.
\end{align}
 The graph $G_{\calB}$ is well-defined since $\sum_{i=1}^k p_{b_i} \le k\times n/k = n$. 
Note that $C(G_{\calB};K_k) = \prod_{i=1}^k p_{b_i}$.
Also, since $p_i$, $i \in [m]$, are prime numbers, there is a unique bijection
\begin{align}
\calB \overset{\varphi}{\longleftrightarrow}C(G_{\calB};K_k).
\end{align}
Therefore, 
\begin{align}
  \Big |\{C(G;K_k) : G \in \G_n \} \Big |& \ge  \Big |\{C(G_{\calB};K_k) : \calB \subseteq [m], |\calB| = k\}  \Big |\\
  & = {m \choose k}  \\&\ge
  \frac{(m-k)^k}{k!}  \\&\ge
 \frac{(cn/(k\log(n/k))-k)^k}{k!}.
 \end{align}

\section{Relationship to the Reconstruction Conjecture}\label{sec:reconstruction}

Theorem \ref{th2} provides a universality  result for RNP-GNNs. Here, we note that the proposed method is closely  related to the reconstruction conjecture, an old open problem in graph theory. This motivates us to explain their relationship/differences. First, we need a  definition for unattributed graphs. 

\begin{definition}
Let $\calF_n \subseteq \G_n$ be a set of graphs and let  $G_v = G(\calV \setminus \{v\})$ for any  finite simple graph $G=(\calV,\calE)$, and  any $v \in \calV$. Then, we say the set $\calF$ is reconstructible if and only if there is a bijection 
\begin{align}
\lb G_v: v \in \calV \rb  \overset{\Phi}{\longleftrightarrow} G,
\end{align}
for any $G \in \calF_n$.  In other words, $\calF_n$ is reconstructible, if and only if  the multi-set $\lb  G_v: v \in \calV\rb $ fully identifies  $G$ for any $G \in \calF_n$. 
\end{definition} 

It is known that the class of disconnected graphs, trees, regular graphs, are reconstructible \citep{kelly1957congruence, mckay1997small}.  The general case is still open; however it is  widely believed that it is true. 

\begin{conjecture}[\cite{kelly1957congruence}] $\G_n$ is reconstructible. 
\end{conjecture}

For RNP-GNNs, the reconstruction from the subgraphs $G^*_v$, $v \in [n]$ is possible, since we relabel any subgraph (in the definition of  $X^*$)
 and this preserves the critical information for the recursion to the original graph. In the reconstruction conjecture, this part of information is missing, and this makes the problem difficult. Nonetheless, since in RNP-GNNs we preserve the original node's information in the subgraphs with relabeling,  the reconstruction conjecture is not required to hold  to show the   universality results for RNP-GNNs, although that conjecture is a  motivation for this paper. Moreover, if it can be shown that the reconstruction conjecture it true, it may be also possible to find a simple encoding of subgraphs to an original graph and this may lead to more powerful but less complex new GNNs.

\section{The RNP-GNN Algorithm }\label{psc}
In this section, we provide pseudocode for RNP-GNNs. The algorithm below computes node representations. For a graph representation, we can aggregate them with a common readout, e.g., $h_G \gets \text{MLP}\Big( \sum_{v \in \calV} h_v^{(k)}  \Big)$.
Following \citep{xu2018powerful}, we use sum pooling here, to ensure that we can represent injective aggregation functions.

%In this section, we explain RNP-GNNs  in an algorithmic way. To do that, we use the results in \citep{xu2018powerful} to convert a multi-set function to MLPs without limiting the expressive power, if MLPs are allowed to be high-dimensional. For initialization, we use one-hot encoded vectors as suggested in \citep{xu2018powerful}.

\begin{algorithm} [H]
\caption{ Recursive Neighborhood Pooling-GNN (RNP-GNN)} 
\begin{algorithmic}
\REQUIRE $G= (\calV,\calE, \{x_v\}_{v \in \calV})$ where $\calV=[n] $, recursion parameters $r_1,r_2,\ldots,r_t\in \N$, $\epsilon^{(i)} \in \R$, $i \in[\tau]$, node features $\{x_v\}_{v \in \calV} $.
\ENSURE $h_v$ for all $v \in \calV$
\STATE $h_v^{\text{in}} \gets x_v$ for all $v \in \calV$
%\STATE $i \leftarrow 1$
%\WHILE{$i<k$} 
%\STATE $i \leftarrow i+1$
\IF{$\tau=1$}
\STATE $$h_v \gets \text{MLP}^{(\tau,1)} \Big(  (1+ \epsilon^{(1)})h^{\text{in}}_v +   \sum_{u \in \calN_{r_1}(v)\setminus \{v\}} \text{MLP}^{(\tau,2)}(h^{\text{in}}_u, \ind{ (u,v) \in \calE})  \Big),$$ for all $v \in \calV$.
\ELSE
\FOR{ all $v \in V$}
\STATE $G_v' \gets G(\calN_{r_1}(v)\setminus \{v\})$, which has node attributes  $\{(h_u^{\text{in}}, \ind{(u,v) \in \calE})'\}_{u \in \mathcal N_{r_1}(v) \setminus \{v\}}$
\STATE $\{\hat{h}_{v,u}\}_{u \in G_v'\setminus \{v\}} \leftarrow \text{RNP-GNN}(G_v', (r_2,r_3,\ldots,r_\tau), (\epsilon^{(2)}, \ldots, \epsilon^{(\tau)}))$
\STATE $h_v \gets \text{MLP}^{(\tau)} \Big ( (1+\epsilon^{(\tau)})h_v^{\text{in}}+\sum_{ u \in  \calN_{r_1}(v)\setminus \{v\}  } \hat{h}_{u,v}\Big).$
\ENDFOR
%\STATE Run a new  RNP-GNN with recursion parameters $(r_2,r_3,\ldots,r_t)$ on $G_v^{(t-1)}(\calN_{r_1}(v)\setminus \{v\})$, $v \in \calV$. 
%    For any $v \in \calV$, the algorithm returns $\hat{h}_{u,v}$ for all $u\in \calN_{r_1}(v)\setminus \{v\}$. Let 
%$$h_v^{(t)} \gets \text{MLP}^{(t)} \Big ( (1+\epsilon^{(t)})h_v^{(t-1)}+\sum_{ u \in  \calN_{r_1}(v)\setminus \{v\}  } \hat{h}_{u,v}\Big).$$
\ENDIF
%\ENDWHILE
\RETURN $\{h_v\}_{v \in \calV}$
%$h_G \gets \text{CONCATENATION}\Big( h_v^{(t)} : v \in \calV \Big)$ or $h_G \gets \text{MLP}^{(t,2)}\Big( \sum_{v \in \calV}\text{MLP}^{(t,1)} \big( h_v^{(t)}\big)  \Big)$, 
% if the final representation must be permutation-invariant  \citep{xu2018powerful, zaheer2017deep}.
\end{algorithmic}
\end{algorithm}

With this algorithm, one can achieve the expressive power of RNP-GNNs if high dimensional MLPs  are allowed
\citep{xu2018powerful, hornik1989multilayer,hornik1991approximation}. That said, in practice, smaller MLPs may be acceptable \citep{xu2018powerful}.

\section{Computing a Covering Sequence} 

As we explained in the context of Theorem \ref{th1}, we need a covering sequence (or an upper bound to that) to design an RNP-GNN network that can count a given substructure. A covering sequence can be constructed from a spanning tree of the graph.

For reducing complexity, it is desirable to have a covering sequence with minimum $r_1$ (Theorem \ref{th3}). Here, we suggest an algorithm for obtaining such a covering sequence, shown in Algorithm~\ref{alg:cov}. For obtaining merely an aribtrary covering sequence, one can compute any minimum spanning tree (MST), and then proceed as with the MST in Algorithm~\ref{alg:cov}.

Given an MST, we build a vertex covering sequence by iteratively removing a leaf $v_i$ from the tree and adding the respective node $v_i$ to the sequence. This ensures that, at any point, the remaining graph is connected. At position $i$ corresponding to $v_i$, the covering sequence contains the maximum distance $r_i$ of $v_i$ to any node in the remaining graph, or an upper bound on that. For efficiency, an upper bound on the distance can be computed in the tree.

To minimize $r_1 = \max_{u \in \calV} d(u,v_1)$, we need to ensure that a node in $\arg\min_{v \in \calV}\max_{u \in \calV} d(u,v)$ is a leaf in the spanning tree. Hence, we first compute $\max_{u \in \calV} d(u,v)$ for all nodes $v$, e.g., by running All-Pairs-Shortest-Paths (APSP) \citep{kleinberg2006algorithm}, and sort them in increasing order by this distance. Going down this list, we try whether it is possible to use the respective node as $v_1$, and stop when we find one.

Say $v^*$ is the current node in the list. To compute a spanning tree where $v^*$ is a leaf, we assign a large weight to all the edges adjacent to $v^*$, and a very low weight to all other edges. If there exists such a tree, running an MST with the assigned weights will find one. Then, we use $v^*$ as $v_1$ in the vertex covering sequence. This algorithm runs in polynomial time.

\begin{algorithm} [H]
\caption{ Computing a covering sequence with minimum $r_1$ \label{alg:cov}} 
\begin{algorithmic}
\REQUIRE $H= (\calV,\calE,X)$ where $\calV=[\tau+1] $
\ENSURE A minimal covering sequence $(r_1,r_2\ldots,r_\tau)$, and its corresponding vertex covering sequence  $(v_1,v_2,\ldots,v_{\tau+1})$  
\STATE For any $u,v \in \calV$, compute $d(u,v)$ using APSP
\STATE $(u_1,u_2,\ldots,u_{\tau+1}) \gets$  all the vertices sorted increasingly in   $s(v):=\max_{u \in \calV} d(u,v)$ 
\FOR {$i = 1$ to $\tau+1$}
\STATE Set edge weights $w({u,v}) = 1 +  \tau \times \ind \{ u = u_i \lor v = u_i\}$ for all $(u,v) \in \calE$ 
\STATE ${H_T \gets}$ the MST of $H$ with weights $w$
\IF{$u_i$ is a leaf in $H_T$} 
\STATE ${v_1 \gets u_i}$
\STATE $r_1 \gets s(u_i)$
\STATE \bf{break}
\ENDIF 
\ENDFOR
%\STATE $v_1 \gets v$
%\STATE $\tilde{H} \gets H(\calV \setminus \{v_1\})$
%\STATE $i \leftarrow 1$
%\WHILE{$i<k$} 
%\STATE $i \leftarrow i+1$
\FOR{$i=2$ to $t+1$}
\STATE $v_i \gets$ one of the leaves of $H_T$
\STATE $r_i \gets \max_{u \in \calV_{H_T}} d(u,v_i)$
\STATE $H_T \gets H_T$ after removing $v_i$
\ENDFOR 
%\STATE Run this algorithm for $\tilde{H}$, and compute $(r_2,r_3,\ldots,r_t)$ and $(v_2,v_3,\ldots,v_{t+1})$ as the minimal covering sequence and the corresponding vertex covering sequence of $\tilde{H}$, respectively
\RETURN $(r_1,r_2,\ldots,r_\tau)$ and $(v_1,v_2,\ldots,v_{t+1})$ 
%\STATE Choose a leaf of $H_T$; call it $v_1$
%\STATE Remove $v_1$ from $H_T$; call the resulting tree $H_{T'}$
%\STATE Apply the same procedure to $H_{T'}$ to achieve a vertex covering sequence  $(v_1,v_2,\ldots,v_{t+1})$  
%\STATE Compute $r_1 = \max_{v \in [t+1]}d(v_1,v)$, and iteratively compute $r_i$, $i=2,\ldots,t$ according to Equation \ref{eq::ref::alg2}.
\end{algorithmic}
\end{algorithm}

\section{Experimental Details}\label{appendix:experiments}

\subsection{Dataset and Task Details}

For the counting experiments, we follow the setup of \cite{chen2020can}. There are two datasets: one consisting of 5000 Erd\H{o}s-Renyi graphs~\citep{erdos1960evolution} and one consisting of 5000 noisy random regular graphs~\citep{steger1999generating}. Each Erd\H{o}s-Renyi graph has 10 nodes, and each random regular graph has either 10, 15, 20, or 30 nodes. Also, $n$ random edges are deleted from each random regular graph, where $n$ is the number of nodes.

For the experiments on distinguishing non-isomorphic graphs, we use the EXP dataset \citep{abboud2021surprising}. This dataset consists of 600 pairs of graphs (so 1200 graphs in total), where each pair is 1-WL equivalent but distinguishable by 3-WL, and each pair contains one graph that represents a satisfiable formula and one graph that represents an unsatisfiable formula. We report the mean and standard deviation across 10 cross-validation folds.

\subsection{RNP-GNN Implementation Details}

Here, we detail some specific design choices we make in implementing our RNP-GNN model. Most embeddings are computed in $\mathbb{R}^d$ for some fixed hidden dimension $d$. The input node features are first embedded in $\mathbb{R}^d$ by an initial linear layer. Then RNP layers are applied to compute node representations. Finally, a sum pooling across nodes followed by a final MLP is used to compute a graph-level output.

An RNP layer for $r=(r_1, \ldots, r_t)$ is implemented as follows. Note that the input node features to this layer are in $\mathbb{R}^d$ due to our initial linear layer. Also, note that we concatenate an extra feature dimension due to the augmented indicator feature at each recursion step. To align these feature dimensions, for $l \in [t]$, we parameterize the $l$-th GIN~\citep{xu2018powerful} by a feedforward neural network $\mathrm{MLP}^{(l)}: \mathbb{R}^{d+l} \to \mathbb{R}^{d+l-1}$. For instance, the last GIN has a feedforward network $\mathrm{MLP}^{(t)}: \mathbb{R}^{d+t} \to \mathbb{R}^{d+t-1}$, because after $t$ levels of recursion we have augmented $t$ features. Dropout and nonlinear activation functions are only applied in the MLPs.

\subsection{Hyperparameters}

For all baseline models, we take the results from other papers. Thus, for the counting experiments the configurations for the baseline models are from \cite{chen2020can}, while for the EXP experiments the configurations for the baseline models are from \cite{abboud2021surprising}.

\textbf{RNP-GNN hyperparameters.} For all experiments we ran random search over hyperparameters. In all cases we used the Adam optimizer with initial learning rate in $\{.01, .001, .0001, .0005\}$.  We train for 100 epochs with a batch size in $\{16, 32, 128\}$. The number of stacked RNP-GNNs for computing node representations is in $\{1, 2\}$. We use a dropout ratio in $\{0, .1, .5\}$. The recursion parameters used varies for each task. We used two layers for each MLP used in the aggregation function. Also, the graph-level output obtained after sum-pooling across nodes is computed by a two layer MLP.

Specifically for the counting experiments, 
 the number of hidden dimensions is searched in $\{16, 32, 64\}$.  For all tasks we used $r_1 = 1$. We use ReLU activations in the MLPs. We either decay the learning rate by half every $25, 50,$ or $\infty$ epochs (where $\infty$ means never decaying).
 
 For the EXP experiments, the number of hidden dimensions is searched in $\{8, 16, 32, 64\}$. We use either ELU or ReLU activations in the MLPs. We decay the learning rate by half at the 50th epoch. The recursion parameters are $r=(2, 1)$.

\end{document}